\documentclass{article}

\usepackage{PRIMEarxiv}

\usepackage[utf8]{inputenc} 
\usepackage[T1]{fontenc}    
\usepackage{hyperref}       
\usepackage{url}            

\usepackage{booktabs}       
\usepackage{amsfonts}       
\usepackage{nicefrac}       
\usepackage{microtype}      
\usepackage{lipsum}
\usepackage{fancyhdr}       
\usepackage{graphicx}       
\graphicspath{{media/}}     
\usepackage{times}  
\usepackage{helvet}  
\usepackage{courier}  
\usepackage{natbib}  
\usepackage{caption} 
\usepackage[linesnumbered,ruled,vlined]{algorithm2e}

\usepackage{newfloat}
\usepackage{listings}
\DeclareCaptionStyle{ruled}{labelfont=normalfont,labelsep=colon,strut=off}
\usepackage{float}
\floatstyle{ruled}
\newfloat{listing}{tb}{lst}{}
\floatname{listing}{Listing}
\usepackage{tikz}
\usepackage{multirow} 
\usepackage{amsthm}
\usepackage{amsfonts}
\usepackage{amsmath}
\usepackage{booktabs}
\usepackage{subcaption}
\usepackage{textcomp}
\usepackage[font=footnotesize, labelfont=footnotesize]{caption}
\newtheorem{definition}{Definition}
\newtheorem{example}{Example}
\newtheorem{lemma}{Lemma}
\newtheorem{theorem}{Theorem}
\newcommand{\hh}[1]{\textcolor{black}{{#1}}}
\newcommand{\cirnum}[1]{\tikz[baseline=(char.base)]{
\node[shape=circle,draw,inner sep=0.5pt] (char) {\scriptsize #1};
}}

\newcommand{\cirnumnormal}[1]{\tikz[baseline=(char.base)]{
\node[shape=circle,draw,inner sep=0.6pt] (char) {\footnotesize #1};
}}

\usepackage{soul}

\newcommand{\prob}[1]{\textcolor{black}{{#1}}}
\usepackage{bibentry}
\usepackage{placeins}
  
\title{Heterogeneous Attributed Graph Learning via Neighborhood-Aware Star Kernels}

\author{
 Hong Huang, Haiming Chen, Hang Gao \\
  Key Laboratory of System Software (Chinese Academy of Sciences) and State Key Laboratory of Computer Science \\
 Institute of Software, Chinese Academy of Sciences \\
  Beijing, China\\
  \texttt{\{huanghong, chm\}@ios.ac.cn} \\
  \texttt{gaohang@iscas.ac.cn} \\
   \And
  Chengyu Yao \\
  National Key Laboratory of Space Integrated Information System \\
  Institute of Software, Chinese Academy of Sciences \\
  Beijing, China\\
  \texttt{yaochengyu2023@iscas.ac.cn} \\
}
\begin{document}
\maketitle

\begin{abstract}
Attributed graphs, typically characterized by irregular topologies and a mix of numerical and categorical attributes, are ubiquitous in diverse domains such as social networks, bioinformatics and cheminformatics. While graph kernels provide a principled framework for measuring graph similarity, existing kernel methods often struggle to simultaneously capture heterogeneous attribute semantics and neighborhood information in attributed graphs. In this work, we propose the \textbf{Neighborhood-Aware Star Kernel (NASK)}, a novel graph kernel designed for attributed graph learning. NASK leverages an exponential transformation of the Gower similarity coefficient to jointly model numerical and categorical features efficiently, and employs star substructures enhanced by Weisfeiler–Lehman iterations to integrate multi-scale neighborhood structural information. We theoretically prove that NASK is positive definite, ensuring compatibility with kernel-based learning frameworks such as SVMs. Extensive experiments are conducted on eleven attributed and four large-scale real-world graph benchmarks. The results demonstrate that NASK consistently achieves superior performance over sixteen state-of-the-art baselines, including nine graph kernels and seven Graph Neural Networks.
\end{abstract}

\section{Introduction}
As a common type of graph~\cite{cai2023efficient, gao2020time, zhang2019efficient}, attributed graphs consist of nodes and edges with multidimensional attributes, thereby capturing rich structural and attribute information. Owing to these expressive representations, attributed graphs have been widely applied in various domains, including cheminformatics~\cite{li2019deepchemstable}, bioinformatics~\cite{he2017misaga,danaci2018eclerize}, and social networks~\cite{bizer2023linked}. 

Given their popularity, recent years have witnessed a surge in research on effective learning techniques for attributed graphs. Among these techniques, there are two mainstream approaches: Graph Neural Networks (GNNs)~\cite{kipf2016semi} and graph kernel methods~\cite{haussler1999convolution}. Both approaches exhibit relatively distinct characteristics~\cite{morris2023weisfeiler}. Compared to GNNs, graph kernels offer greater stability due to their stronger interpretability, and are capable of capturing richer neighborhood information, as GNN architectures are theoretically no more powerful than the 1-Weisfeiler–Lehman (1-WL) test in distinguishing non-isomorphic graphs~\cite{morris2017glocalized}. However, graph kernel methods often struggle to effectively model multidimensional semantic information and tend to incur higher computational costs when applied to large-scale and high-dimensional datasets~\cite{liu2024exploring}. 

Many efforts in graph kernel research have aimed to enhance the ability to capture rich attribute semantics. Some methods designed for attributed graphs rely on predefined substructures. However, a subset of these methods~\cite{mahe2004extensions,kriege2012subgraph,orsini2015graph} often suffer from high computational complexity~\cite{salim2022neighborhood}, while others~\cite{feragen2013scalable,da2017tree} lack well-defined node correspondence and fail to effectively incorporate node and edge label information, thus overlooking crucial structural differences.

In recent years, substantial efforts have been devoted to learning heterogeneous attributes that contain a mix of categorical and numerical features.  \prob{Neumann \textit{et al}~\cite{neumann2012efficient} and Morris \textit{et al.}~\cite{morris2016faster} enable kernel computation on attributed graphs by applying discretization techniques, which, despite relatively scalable, often compromise the fidelity of attribute semantics due to the transformation of original attribute values}. Togninalli \textit{et al}.~\cite{togninalli2019wasserstein} proposed the WWL kernel, which utilizes the Wasserstein distance based on node feature representations to model heterogeneous attributes. To mitigate the high cost of Wasserstein distance, Isolation Graph Kernel~\cite{xu2021isolation} uses a distribution-based kernel via kernel mean embedding, offering better accuracy and faster runtime. Nevertheless, both methods lack formal guarantees of positive definiteness when applied to numerical features, which may affect robustness. Salim \textit{et al}.~\cite{salim2022neighborhood} proposed a kernel method defined as a weighted sum of two kernels: an R-convolution kernel that captures attribute information, and an optimal assignment kernel that models label information. However, this direct combination still inevitably results in the loss of important semantic information, particularly when dealing with categorical and numerical features simultaneously.

Motivated by aforementioned works, we propose the \textbf{Neighborhood-Aware Star Kernel (NASK)}—a novel positive definite graph kernel that effectively captures both heterogeneous attribute semantics and \hh{neighborhood} structural information in attributed graphs. First, we design an attribute similarity function $\mathsf{P}$ based on an exponential transformation of the Gower similarity coefficient~\cite{gower1971general}, \hh{which preserves its ability to jointly model numerical and categorical features, while maintaining lower computational costs compared to one-hot encoding~\cite{lautrup2025syntheval}}. 
Furthermore, we prove the positive definiteness of $\mathsf{P}$, demonstrating its stability and generalizability. Second, our design employs star substructures (i.e., star subgraphs) as local structural information. These structures are prevalent in real-world attributed graphs and play a crucial role in capturing high-level semantics~\cite{liu2022cspm,qin2024colorful}. Additionally, all star substructures counting can be performed in linear time~\cite{ahmed2017graphlet}, ensuring the scalability of the method. To avoid limiting the kernel to purely fixed substructure, we incorporate the WL algorithm~\cite{Weisfeiler1968} into NASK, enabling iterative aggregation to incorporate neighborhood structural information.

Moreover, we theoretically prove that NASK is positive definite. This guarantees that it induces a valid Reproducing Kernel Hilbert Space (RKHS), enabling the use of kernel-based learning methods such as Support Vector Machines (SVMs). The positive definiteness further ensures that the kernel captures meaningful similarity while preserving the mathematical properties required for convergence and generalization in kernel learning frameworks. Experimental results demonstrate that our kernel empirically outperforms state-of-the-art baselines on various real-world attributed graph benchmarks.

In summary, our contributions are as follows. \begin{itemize} \item We propose a novel graph kernel NASK, that captures neighborhood structural semantics while effectively handling heterogeneous attributes in attributed graphs. 

\item We theoretically establish the positive definiteness of NASK, guaranteeing that it induces a valid RKHS.  

\item We apply NASK within a SVM framework and conduct extensive experiments on multiple real-world attributed graph benchmarks, demonstrating that NASK consistently outperforms state-of-the-art baselines in terms of accuracy and robustness, validating its practical effectiveness and broad applicability.

\end{itemize}

\section{Related Work}
Kernels for attributed graphs have received increased attention recently, and research efforts have focused on different types of attributes. Early methods~\cite{gartner2003graph, borgwardt2005protein, da2012tree, kashima2003marginalized} extended random walk and shortest path kernels~\cite{borgwardt2005shortest} to handle attribute graphs. However, they are computationally costly and limited structural expressiveness. Kriege \textit{et al}.~\cite{kriege2022weisfeiler} revisits random walk kernels and shows they can effectively handle attributed graphs when guided by the Weisfeiler-Lehman framework. Their reliance on discrete label matching makes them unsuitable for continuous attributes. Several subsequent approaches~\cite{kondor2016multiscale, kriege2012subgraph, orsini2015graph, feragen2013scalable} address this limitation by comparing vertex sets using both structural and attribute similarities. However, the kernel methods lack targeted consideration for numerical attributes.

To improve scalability, several frameworks~\cite{neumann2012efficient, morris2016faster} discretize numerical attributes and apply kernels designed for categorical attributes. For example, Hash Graph Kernel~\cite{morris2016faster} converts the attributes to categorical ones using well-defined hashing functions. Although the above approaches are scalable, it is coming at the cost of transformation of the original attributes. To address this, recent efforts have explored unified frameworks that jointly model structure and attributes while preserving attribute semantics. These include distributional approaches~\cite{togninalli2019wasserstein, xu2021isolation, perez2024gaussian}, regularized kernels, and Hilbert space embeddings~\cite{wijesinghe2021regularized, zhang2018retgk}. While these methods represent significant progress, they still face limitations in offering theoretical guarantees for positive definiteness and in comprehensively integrating structural and attribute information. Meanwhile, hierarchical and $k$-WL based methods~\cite{bai2022hierarchical, bai2024qbmk, tang2024deep, morris2017glocalized} enhance structural expressiveness, yet lack targeted modeling for attribute semantics. Furthermore, domain-specific kernels~\cite{huang2021density, griffiths2023gauche, yao2024molecular} highlight the need for general graph kernels that can effectively handle heterogeneous attributed graphs.

\section{Preliminaries}
For clarity and consistency, this section outlines the definitions of the mathematical notations and fundamental concepts used throughout this paper. 
\subsection{Key Concept about Graphs}
We use $\mathcal{Q} = \{q_0, q_1, \ldots, q_{|\mathcal{Q}|-1}\}$ to denote a set, where $|\mathcal{Q}|$ represents its cardinality (i.e., the number of elements in $\mathcal{Q}$), and $q_i$ denotes the $i$-th element of $\mathcal{Q}$. A vertex is denoted by $v$, and $\mathcal{N}(v)$ denotes the set of vertices adjacent to $v$. Let \( \mathcal{X} \) be the set of all graphs under consideration. $\mathbb{R}$ denotes the set of real numbers. The degree of $v$ is defined as the size of its neighborhood, i.e., $|\mathcal{N}(v)|$.
\begin{definition}[Attributed Graph]\label{Attributed Graph}
 An attributed graph is a tuple $\mathrm{G} = \langle\mathrm{A}, \mathrm{V}, \mathrm{E}, \lambda, \mathrm{F}\rangle$, where $A$ is a set of attribute names, $V$ is a set of vertices (also called nodes), $E$ is a set of edges, $F$ is a value set \hh{that contains a mix of numerical and categorical values}, $\lambda$:$A \times (V \cup E) \mapsto F $ is a function that maps vertex-attribute or edge-attribute pair to corresponding attribute values. 
\end{definition}
\begin{example}
\upshape An attributed graph is illustrated as $\mathrm{G}$ in Figure~\ref{fig:framework}, consisting of nine nodes (e.g., A, $\ldots$, J) in $\mathrm{V}$, nine edges in $\mathrm{E}$ represented by black lines, and the corresponding attribute information shown in the table highlighted in red. Each node in $\mathrm{G}$ is associated with one numerical attribute (\textit{Age}) and two categorical attributes (\textit{Gender} and \textit{Region}), all belonging to the set of attribute names $\mathrm{A}$. The table specifies the mapping function $\lambda$, assigning each node and edge their respective attribute values in $\mathrm{F}$; for example, node A is assigned the values \{\textit{23, Male, Bratislava}\}. 
\end{example}
\begin{definition}[Attributed Star Graph] \label{definition 2}
An attributed star $\mathrm{S}$ is formally defined as a tuple $\mathrm{S} = \langle\mathrm{A}, \mathrm{V}, \mathrm{E}, \lambda, \mathrm{F}, \mathrm{C}, \mathrm{L}\rangle$, where the first five elements are defined in Definition~\ref{Attributed Graph}, $\mathrm{C}$ is the core node, and $\mathrm{L}$ is the set of leaf nodes.
\end{definition}
\begin{example}
\upshape A star graph is shown as $\mathrm{S}_1$ in Figure~\ref{fig:framework} with a center node $A$ and its five leaf nodes $\{B,C,D,E,F\}$.
\end{example}

\subsection{Gower Similarity Coefficient}\label{sec: similarity}
Gower similarity coefficient~\cite{gower1971general} is a metric designed to measure pairwise similarity between data instances that may contain both numerical and categorical  attributes. It is particularly well-suited for handling heterogeneous attributes.
\begin{definition} [Gower similarity coefficient]\label{definition: gower}
Given two objects $\mathbf{x}$ and $\mathbf{x}'$ described by d-dimensional attributes, the Gower similarity coefficient $k(\mathbf{x}, \mathbf{x}')$ is defined as the average of the partial similarities computed on each individual attribute:
\begin{equation}
k(\mathbf{x}, \mathbf{x}') = \frac{1}{D} \sum_{d=1}^{D} s_d(x_d, x'_d),
\end{equation}
where $x_d$ and $x'_d$ are the values of the $d$-th attribute in $\mathbf{x}$ and $\mathbf{x}'$, respectively, and $s_d(\cdot,\cdot)$ is the partial similarity function specific to the type of attribute $d$.
\end{definition}
\begin{definition}[Normalized numerical similarity (Gower)]\label{numerical}
For numerical attributes, the partial similarity is typically computed as:
\begin{equation}\label{Eq.3}
s_d(x_d, x'_d) = 1 - \frac{|x_d - x'_d|}{\text{range}_d},
\end{equation}
where $\text{range}_d = \max_d - \min_d$, with $\max_d$ and $\min_d$ denoting the maximum and minimum values of the $d$-th attribute, respectively.
\end{definition}
\begin{definition}[Normalized categorical  similarity (Gower)]\label{categorized}
For categorical  attributes, the partial similarity is:
\begin{equation}\label{Eq.4}
s_d(x_d, x'_d) =
\begin{cases}
1, & \text{if } x_d = x'_d \\
0, & \text{otherwise.}
\end{cases}
\end{equation}
\end{definition}
Note that if $d \neq d'$, then $s_d(x_d, x_{d'}') = 0$, indicating that values corresponding to different dimensions of attributes are considered entirely dissimilar.

\begin{figure*}[t]
 \footnotesize
\centering
\includegraphics[width=1\textwidth]{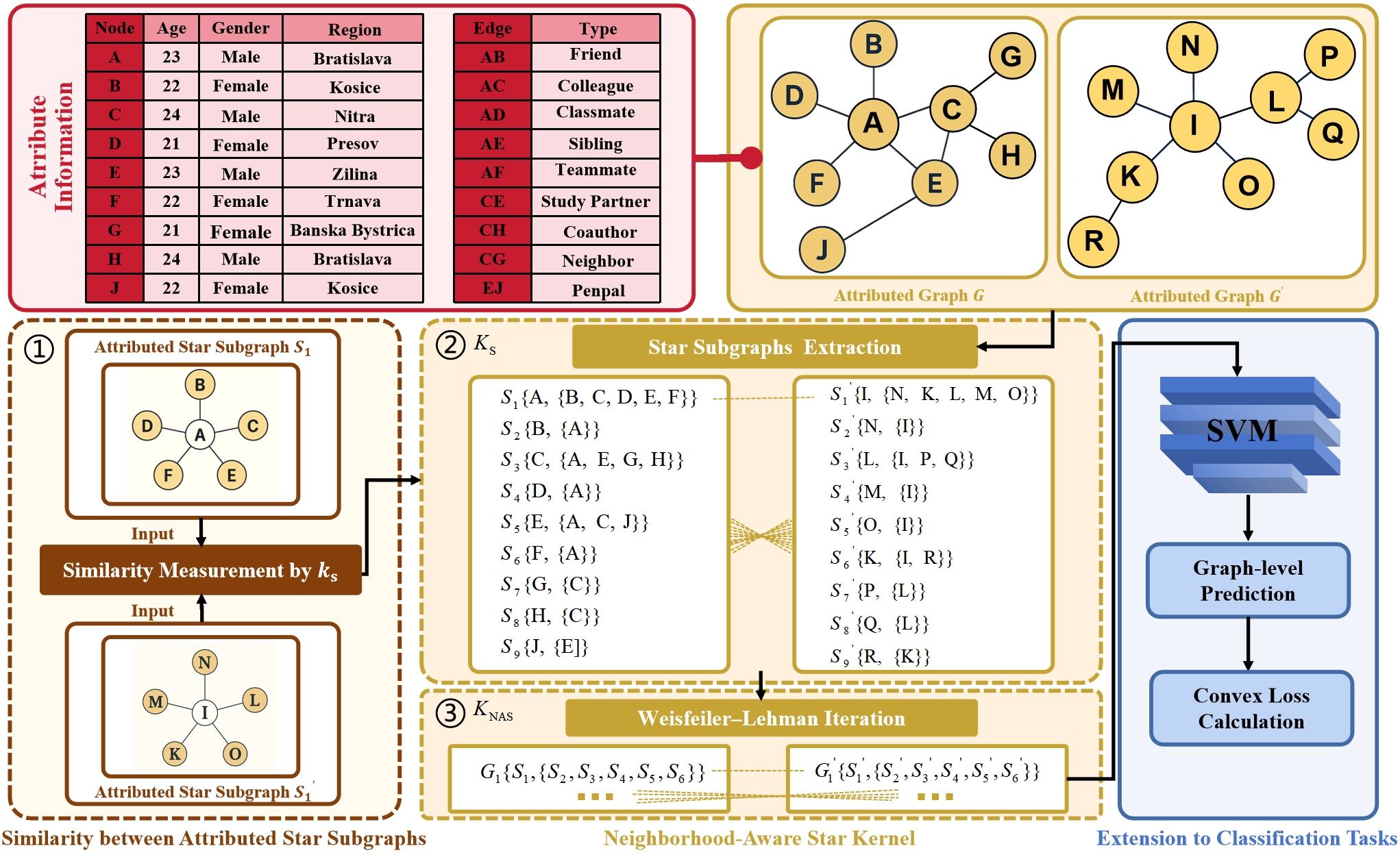} 
\caption{
 \footnotesize Overview of our proposed NASK. Given two attributed graphs $\mathrm{G}$ and $\mathrm{G}'$ in the yellow box. The attribute information of the graph $\mathrm{G}$ is presented in the red box. Step \protect\cirnum{1} illustrates the similarity measurement between attributed star subgraphs. Step \protect\cirnum{2} demonstrates the construction of the \hh{graph kernel} $K_\mathrm{S}$ based on these star subgraphs, denoted as $S_i$. Finally, Step \protect\cirnum{3} defines the NASK, formally denoted as $K_{\mathrm{NAS}}$, by integrating the methods from the previous two steps along with WL algorithm to expand each star subgraph with its one-hop neighborhood. For example, the expanded star graph from $S_1$ is represented as $\{S_1, \{S_2, S_3, S_4, S_5, S_6\}\}$. We present a simple example for illustration, though our method is general and not limited to this case.
}
\label{fig:framework}
\end{figure*}

\section{Methods}
This section introduces the core concepts behind our proposed graph kernel NASK. As illustrated in Figure~\ref{fig:framework}, NASK takes two attributed graphs as inputs. First, we define a kernel $k_\mathrm{s}$ to measure similarity between two star subgraphs. Next, all star subgraphs are extracted from the input graphs, and a \hh{graph kernel} $K_\mathrm{S}$ is constructed based on the pairwise similarity between these subgraphs measured by $k_\mathrm{s}$. Finally, the star subgraphs 
are expanded with multi-hop neighborhood information via WL algorithm, resulting in NASK, which is subsequently employed in a SVM for graph classification.

The details of our method are introduced in three parts which are in sequence:
$\textbf{(\romannumeral1)}$ Similarity measurement between attributed star subgraphs;
$\textbf{(\romannumeral2)}$ \hh{Star kernel;}
$\textbf{(\romannumeral3)}$ Neighborhood-Aware Star Kernel.

\subsection{Similarity Measurement between Attributed Star Subgraphs}\label{4.1}
The procedure illustrated in step \cirnumnormal{1} of Figure~\ref{fig:framework} is detailed in this section. Given two attributed star subgraphs $\mathrm{S} = \langle\mathrm{A}, \mathrm{V}, \mathrm{E}, \lambda, \mathrm{F}, \mathrm{C}, \mathrm{L}\rangle$ and $\mathrm{S}' = \langle\mathrm{A}', \mathrm{V}', \mathrm{E}', \lambda', \mathrm{F}', \mathrm{C}', \mathrm{L}'\rangle$, our proposed kernel $k_\mathrm{s}$  measures their similarity using a positive definite similarity function $\mathsf{P}$ derived from the Gower similarity coefficient. The proposed method consists of three components: first, we define a positive definite function $s_d'$, which is tailored to a single attribute $d$; second, we construct the attribute similarity function $\mathsf{P}$ for nodes and edges with multi-dimensional attributes based on $s_d'$; and finally, we introduce a kernel method $k_\mathrm{s}$ to measure the similarity between two star subgraphs.

$s_d(x_d,x_d')$ is a symmetric similarity function for the $d$-th attribute, as defined in Definitions~\ref{numerical} and \ref{categorized}. Based on this, we derive an exponentially transformed similarity function $s_d'$ for the $d$-th attribute:
\begin{equation}
    s_d'(x_d, x_d') = \exp\left( -\gamma (1 - s_d(x_d, x_d')) \right).
\end{equation}
where $\gamma > 0$ is a scaling parameter. We have that $s_d'$ defines a positive definite function under the condition that $1 - s_d(x_d, x_d')$ is a conditionally negative definite (CND) function. The transformed similarity $s_d’(x_d, x_d’)$ preserves the monotonicity of $s_d$ while maintaining its strong capability to handle heterogeneous attributes.

\begin{lemma}\label{lemma1}
Let $s_d(x_d, x_d')$ be a symmetric similarity function defined on the $d$-th attribute. Define $f_d(x_d, x_d') := 1 - s_d(x_d, x_d')$. Then \hh{under the constructions of $s_d$ for numerical, categorical and binary attributes in Gower similarity coefficient (see details in definition~\ref{definition: gower})}, the function $f_d$ is CND (see \textbf{Appendix} A.1 for the proof).
\end{lemma}

\begin{lemma} \label{lemma2}
\hh{$s_d'$ is positive definite} (see \textbf{Appendix} A.2 for the proof).
\end{lemma}

Given a pair of nodes $(v, v')$ and $v \in \mathrm{V}, v' \in \mathrm{V}'$, we design 
an attribute similarity function $\mathsf{P}(v,v')$ for each pair based on $s_d'$. Specifically, $\mathsf{P}(v, v')$ is given by the following equation:
\begin{equation}\label{EQ:P}
    \mathsf{P}(v, v') = \frac{1}{\mathrm{D}} \sum_{d=1}^{\mathrm{D}} s_d'(\lambda(\mathrm{A}, v)_d, \lambda'(\mathrm{A}', v')_d),
\end{equation}
\hh{where $\mathrm{D}$ denotes the dimensionality} of attributes, and $\lambda(\mathrm{A}, v)_d$ denotes the $d$-th element of the attribute value of $v$. Similarly, $\lambda'(\mathrm{A}', v’)_d$ denotes the $d$-th element of the attribute value of $v’$.  The similarity between a pair of edges $(e, e')$ could also be computed by replacing $v$ and $v'$ with $e$ and $e'$, respectively, in Eq.~\ref{EQ:P}.
\begin{theorem}\label{theorem 1}
\hh{$\;\mathsf{P}$ is positive definite} (see \textbf{Appendix} A.3 for the proof).
\end{theorem}

The kernel $k_\mathrm{s}$ between two star subgraphs $\mathrm{S}$ and $\mathrm{S}'$ is defined by the following equation:
\begin{equation} \label{k_s}
\small
    k_\mathrm{s}(\mathrm{S}, \mathrm{S'}) = \sum_{n \in R^{-1}(\mathrm{S})} \sum_{n' \in R^{-1}(\mathrm{S}')}  \mathsf{P}(\mathrm{C},\mathrm{C'})\mathsf{P}(n, n'), 
\end{equation}
where $R^{-1}(\mathrm{S})$ denotes the set of valid decompositions of $\mathrm{S}$ into nodes and edges, and similarly for $R^{-1}(\mathrm{S}’)$. The function $\mathsf{P}(n, n’)$ measures the attribute similarity between two elements $n \in \mathrm{S}$ and $n’ \in \mathrm{S}’$, where $n$ and $n’$ are either nodes or edges. \hh{To softly emphasize the importance of the similarity between the two center nodes $\mathrm{C}$ and $\mathrm{C'}$, the kernel computes the sum of the products $\mathsf{P}(\mathrm{C}, \mathrm{C’}) \cdot \mathsf{P}(n, n’)$, ensuring that structural similarities are weighted by central node similarity}.

\begin{theorem}\label{theorem 2}
Let $\mathrm{S}$ and $\mathrm{S}'$ be two attributed star subgraphs. The kernel $k_\mathrm{s}(\mathrm{S}, \mathrm{S}')$, as defined in Eq.~\ref{k_s}, is a valid positive definite kernel (see \textbf{Appendix} A.4 for the proof).
\end{theorem}
\subsection{Star Kernel}\label{Local Kernel}
The procedure illustrated in step \tikz[baseline=(char.base)]{
\node[shape=circle,draw,inner sep=1pt] (char) {2};
} of Figure~\ref{fig:framework} is detailed in this section. Given two attributed graphs $\mathrm{G} = \langle\mathrm{A}, \mathrm{V}, \mathrm{E}, \lambda, \mathrm{F}\rangle$ and $\mathrm{G}' = \langle\mathrm{A}', \mathrm{V}', \mathrm{E}', \lambda', \mathrm{F}'\rangle$, we propose a graph kernel $K_\mathrm{S}$ that measures their similarity based on the similarities between their attributed star subgraphs (see Section~\ref{4.1} for details). The proposed method consists of two main components: first, we decompose each input graph into a set of star subgraphs; second, we compute $K_\mathrm{S}$ based on the pairwise similarities between the corresponding star subgraphs. We give a pseudocode of $K_\mathrm{S}$ in Algorithm 1 (see details in \textbf{Appendix} C).

According to Definition~\ref{definition 2}, a star subgraph consists of a center node and its directly connected neighbors. Based on this notion, we define the star subgraph extraction function:
\begin{equation}
\mathcal{F}: v \times \mathrm{G} \rightarrow \mathrm{S}, \;\;\;\;  \mathrm{S} = \langle \mathrm{A}^s, \mathrm{V}^s, \mathrm{E}^s, \lambda^s, \mathrm{F}^s, \mathrm{C}^s, \mathrm{L}^s\rangle,
\end{equation}
which maps a given node $v \in \mathrm{V}$ and an attributed graph $\mathrm{G}$ to the corresponding star subgraph centered at $v$.The extraction operation $\mathcal{F}$ is formally defined as:
\begin{align}\label{it0}
\mathcal{F}(v, \mathrm{G}) = &\big\langle\;
    \mathrm{G}.\mathrm{A},\;
    \{v\} \cup \mathcal{N}(v),\;
    \{(v, u) \in \mathrm{E} \mid u \in \mathcal{N}(v)\},\; \notag \\
    & \mathrm{G}.\lambda, \;
    \mathrm{G}.\mathrm{F},\;
    v,\;
    \mathcal{N}(v)
\;\big\rangle,
\end{align}
where $\mathcal{N}(v)$ denotes the set of neighbors of node $v$, and the edge set $\mathrm{E}^s = \{(v, u) \in \mathrm{E} \mid u \in \mathcal{N}(v)\}$ includes all edges connecting the central node to its neighbors. Each such star subgraph $\mathcal{F}(v, \mathrm{G})$ encapsulates the structure and attribute information around the central node $v$. We denote the set of all star subgraphs in $\mathrm{G}$ as:
\begin{equation}
\hat{\mathcal{S}}(\mathrm{G}) = \{ \mathcal{F}(v, \mathrm{G}) \mid v \in \mathrm{V} \},
\end{equation}
where $V$ is the vertex set of $\mathrm{G}$. The total number of star subgraphs is therefore $|\hat{\mathcal{S}}(\mathrm{G})| = |\mathrm{V}|$.

Given two graphs $\mathrm{G}$ and $\mathrm{G}'$, the graph kernel $K_\mathrm{S}$ is defined as:
\begin{equation}
K_\mathrm{S}(\mathrm{G}, \mathrm{G}') = \sum_{\mathrm{S} \in \hat{\mathcal{S}}(\mathrm{G})} \sum_{\mathrm{S}' \in \hat{\mathcal{S}}(\mathrm{G}')} k_\mathrm{s}(\mathrm{S}, \mathrm{S}'),
\end{equation}
where $k_\mathrm{s}(\mathrm{S}, \mathrm{S}')$ denotes the similarity between two star subgraphs $\mathrm{S}$ and $\mathrm{S}'$ based on their attribute and structural similarities.
\begin{theorem}\label{Theorem2}
Let $\mathrm{G}$ and $\mathrm{G}'$ be two attributed graphs, and let $K_\mathrm{S}(\mathrm{G}, \mathrm{G}')$ be the graph kernel defined as the sum of kernel $k_\mathrm{s}$ over all pairs of star subgraphs. Then $K_\mathrm{S}$ is a valid positive definite kernel (see \textbf{Appendix} A.5 for the proof).
\end{theorem}
\subsection{Neighborhood-Aware Star Kernel}
\label{sec:SWL-kernel}
The procedure illustrated in step \tikz[baseline=(char.base)]{
\node[shape=circle,draw,inner sep=1pt] (char) {3};
} of Figure~\ref{fig:framework} is detailed in this section. Given two attributed graphs $\mathrm{G} = \langle\mathrm{A}, \mathrm{V}, \mathrm{E}, \lambda, \mathrm{F}\rangle$ and $\mathrm{G}' = \langle\mathrm{A}', \mathrm{V}', \mathrm{E}', \lambda', \mathrm{F}'\rangle$, the kernel is designed by two steps: first, we define a expansion function $\mathcal{L}$ for star subgraphs with iterative propagation, following the Weisfeiler--Lehman (WL) algorithm; second, we define $K_{\mathrm{NAS}}$ based on the pairwise similarities between the corresponding subgraphs derived from $\mathcal{L}$. This kernel extends the expressiveness of the star kernel $K_\mathrm{S}$ by incorporating multi-hop neighborhood information. We give a pseudocode of NASK in Algorithm 2 (see details in \textbf{Appendix} C).

We define the iterative expansion of star subgraphs based on the WL algorithm as follows:
\begin{equation}
\mathcal{L}: \mathrm{S}^{h-1} \times \mathrm{G} \rightarrow \mathrm{S}^{h},\;\; \mathrm{S}^h = \langle\mathrm{A}^s, \mathrm{V}^s, \mathrm{E}^s, \lambda^s, \mathrm{F}^s, \mathrm{C}^s, \mathrm{L}^s\rangle,
\end{equation}
where $h$ is the depth of iteration, and $\mathcal{L}$, derived from the $\mathrm{HASH}$ function of WL (see details in Appendix A.6), denotes the operation that extends the star graph $\mathrm{S}^{h-1}$ by including its $1$-hop neighborhood. The initial star subgraph is given by $\mathrm{S}^1 = \mathcal{F}(v, \mathrm{G})$, as defined in Eq.~\ref{it0}. Given a star subgraph $\mathrm{S}^{h-1}$, the function $\mathcal{L}$ operates as follows:
{\footnotesize
\begin{align}
\mathcal{L}(\mathrm{S}^{h-1},\mathrm{G}) = &\big\langle\; 
    \mathrm{G}.\mathrm{A},\;
    \mathrm{S}^{h-1} \cup \mathcal{N}(\mathrm{S}^{h-1}),\; \{(\mathrm{S},\mathrm{S}^{h-1}) \mid \mathrm{S} \in \mathcal{N}(\mathrm{S}^{h-1})\},\notag \\ 
    &  \; 
    \mathrm{G}.\lambda,\;\;
    \mathrm{G}.\mathrm{F},\;\;
    \mathrm{S}^{h-1},\;\;
    \mathcal{N}(\mathrm{S}^{h-1})\;
\big\rangle,
\end{align}}
where $\mathcal{N}(\cdot)$ denotes the set of neighboring star subgraphs of the given star subgraph, and $.$ is used to access elements of the tuple $\mathrm{S}^{h-1}$. The components in the big bracket correspond to the elements in the tuple of $\mathrm{S}^{h}$. For example, the node set $\mathrm{S}^h.\mathrm{V}$ is updated to $\mathrm{S}^{h-1} \cup \mathcal{N}(\mathrm{S}^{h-1})$. In particular, WL iteration can be terminated at a specific depth $H$ for a finite graph where $H$ is bounded by $min(|\mathrm{G}.V|, |\mathrm{G'}.V'|)$~\cite{pellizzoni2025graph}. By expanding the star subgraph $\mathrm{S}^{h-1}$ to incorporate an additional hop of neighboring nodes, the design captures richer positional and structural information.

The set of all $h$-hop expanded star subgraphs in $\mathrm{G}$ is defined inductively via 1-hop expansions as:
\begin{align}
\hat{\mathcal{S}}^{(1)}(\mathrm{G}) &= \hat{\mathcal{S}}(\mathrm{G}), \quad \\
\hat{\mathcal{S}}^{(h)}(\mathrm{G}) &= \left\{ \mathcal{L}(\mathrm{S}^{(h-1)},\mathrm{G}) \mid \mathrm{S}^{(h-1)} \in \hat{\mathcal{S}}^{(h-1)}(\mathrm{G}),\;\; h \ge 2 \right\}.
\end{align}
The Neighborhood-Aware Star Kernel is then defined as:
\begin{equation}
K_{\mathrm{NAS}}^{(H)}(\mathrm{G}, \mathrm{G}') = \sum_{h = 1}^{H}\sum_{\mathrm{S} \in \hat{\mathcal{S}}^{(h)}(\mathrm{G})} \sum_{\mathrm{S}' \in \hat{\mathcal{S}}^{(h)}(\mathrm{G}')} k_\mathrm{s}(\mathrm{S}, \mathrm{S}'),
\end{equation}
\hh{where $k_\mathrm{s}(\mathrm{S}, \mathrm{S}')$ is introduced in Section~\ref{4.1}} and $H$ is the deepest WL iteration. Under the condition that $H=1$, the NAS kernel reduces to the kernel $K_\mathrm{S}$ (w.r.t. $K_{\mathrm{NAS}}^{(1)}(\mathrm{G}, \mathrm{G}') = K_\mathrm{S}(\mathrm{G}, \mathrm{G}')$). As $H$ increases, the kernel captures higher-order structural and attribute interactions, making $K_{\mathrm{NAS}}^{(H)}$ increasingly expressive. \hh{Furthermore, this expressiveness comes at the acceptable computational complexity which is $\mathcal{O}(Hn^2(n+m)^2d)$ in the worse case where n, m denote numbers of nodes and edges of input graphs, and d denotes the dimensionality of attributes (the detailed analysis is provided in \textbf{Appendix} A.8)}.
\begin{theorem}\label{theorem 4}
Let $ K_{\mathrm{NAS}}^{(H)}(\mathrm{G}, \mathrm{G}')$ be the Neighborhood-Aware Star kernel of iteration depth $H$, defined over the space of attributed graphs. Then $K_{\mathrm{NAS}}^{(H)}$ is a valid positive definite kernel (see \textbf{Appendix} A.8 for the proof).
\end{theorem}
\begin{proof}
Since:
\begin{itemize}
  \item $k_\mathrm{s}$ is positive definite.
  \item The set of $h$-hop star subgraphs is finite for finite graphs.
  \item The triple sum is a finite sum of PD kernels.
\end{itemize}
\end{proof}
\paragraph{Extension to Classification Tasks}
The positive definiteness of $K_{\mathrm{NAS}}$ ensures that it induces a valid RKHS, making it compatible with standard kernel-based learning algorithms. In this work, we use Support Vector Machines (SVMs) for graph classification, leveraging the theoretical guarantees of $K_{\mathrm{NAS}}$ to support convergence and generalization. Details are provided in \textbf{Appendix} B.

\begin{table*}[!t]
\setlength{\tabcolsep}{9pt}
\centering
\caption{\footnotesize Classification accuracy on numerical and heterogeneous attributed graphs benchmarks.}
\label{tab: accuracy2}
\small
\begin{tabular}{l cccccc}
\toprule
\textbf{Method} & \textbf{SYNTHETIC} & \textbf{SYNTHIE} & \textbf{ENZYMES} & \textbf{PROTEINS\_full} & \textbf{BZR} & \textbf{COX2} \\
\midrule
PK               & 46.2$\pm$3.6 & 46.2$\pm$3.6 & 21.5$\pm$3.4 & 59.6$\pm$0.2 & 78.8$\pm$5.5 & 63.9$\pm$ 3.0 \\
ML               & 47.7$\pm$7.3 & 49.0$\pm$8.3 & 33.2$\pm$5.8 & 71.1$\pm$4.6 & 81.3$\pm$6.2 & 75.6$\pm$1.6 \\
GH               & 74.3$\pm$5.6 & 73.8$\pm$7.3 & 65.7$\pm$0.8 & 74.8$\pm$0.3 & 76.5$\pm$1.0 & 76.4$\pm$1.4 \\
HGK-WL           & 96.0$\pm$0.3 & 82.0$\pm$0.4 & 63.0$\pm$0.7 & 75.9$\pm$0.2 & 78.6$\pm$0.6 & 78.1$\pm$0.5 \\
WWL              & 86.8$\pm$1.0 & 96.0$\pm$0.5 & 73.3$\pm$0.8 & 77.9$\pm$0.8 & 84.4$\pm$2.0 & 78.3$\pm$0.5 \\
RetGK            & 96.2$\pm$0.3 & 91.7$\pm$1.7 & 72.2$\pm$0.8 & 78.0$\pm$0.3 & 86.4$\pm$1.2 & 80.1$\pm$0.9 \\
RWK              & 74.7$\pm$4.3 & 66.8$\pm$5.1 & 76.4$\pm$4.5 & 79.3$\pm$6.1 & 86.2$\pm$5.6 & 81.2$\pm$5.3 \\
SWWL             & 83.4$\pm$4.3 & 84.6$\pm$6.8 & 66.7$\pm$5.0 & 75.1$\pm$6.0 & 85.4$\pm$4.1 & 78.6$\pm$5.9\\
\midrule
GIN              & 80.5$\pm$6.6 & 89.7$\pm$4.6 & 59.6$\pm$4.5 & 73.3$\pm$4.0 & 85.4$\pm$5.1 & 75.6$\pm$2.3 \\
GraphSAGE        & 88.0$\pm$7.3 & 51.3$\pm$9.9 & 58.2$\pm$6.0 & 73.0$\pm$4.5 & 81.2$\pm$4.2 & 73.9$\pm$1.7 \\
\midrule
\textbf{NASK} & \textbf{97.9$\pm$0.3} & \textbf{97.1$\pm$0.3} & \textbf{78.3$\pm$4.1} & \textbf{81.1$\pm$0.9} & \textbf{88.8$\pm$1.8} & \textbf{82.9$\pm$1.8} \\

\bottomrule
\end{tabular}
\end{table*}

\begin{table}[!b]
\centering
\caption{\footnotesize Classification accuracy on categorical attributed graphs benchmarks.}
\label{tab: accuracy}
\small
\setlength{\tabcolsep}{12pt}
{\renewcommand{\arraystretch}{1.1}
\begin{tabular}{l cccccc}
\toprule
\textbf{Method} & \textbf{MUTAG} & \textbf{NCI1} & \textbf{PTC\_MR} & \textbf{D\&D} & \textbf{ENZYMES} & \textbf{PROTEINS} \\
\midrule
WL-VH            & 86.7$\pm$7.3 & 85.2$\pm$2.2 & 64.9$\pm$6.4 & 78.3$\pm$0.3 & 50.7$\pm$7.3 & 75.0$\pm$0.3 \\
PK               & 76.6$\pm$5.2 & 82.1$\pm$2.1 & 51.7$\pm$3.7 & 77.7$\pm$4.2 & 44.0$\pm$6.3 & 73.1$\pm$4.7 \\
GH               & 82.5$\pm$5.8 & 71.0$\pm$2.3 & 60.2$\pm$9.4 & TIMEOUT & 37.5$\pm$0.8 & 74.8$\pm$2.4 \\
ML               & 87.2$\pm$7.5 & 79.7$\pm$1.8 & 64.5$\pm$5.8 & 78.6$\pm$4.0 & 48.5$\pm$7.8 & 74.2$\pm$4.4 \\
WWL              & 87.3$\pm$1.5 & 85.8$\pm$0.3 & 66.3$\pm$1.2 & 79.7$\pm$0.5 & 59.1$\pm$0.8 & 74.3$\pm$0.6 \\
RetGK            & 90.3$\pm$1.1 & 84.5$\pm$0.2 & 62.5$\pm$1.6 & 81.6$\pm$0.3 & 60.4$\pm$0.8 & 75.8$\pm$0.6 \\
RWK              & 93.6$\pm$3.7 & 85.5$\pm$0.3 & 69.5$\pm$6.1 & 72.7$\pm$4.2 & 64.1$\pm$3.6 & 71.3$\pm$4.1 \\
\midrule
GCN              & 85.6$\pm$5.8 & 80.2$\pm$2.0 & 64.2$\pm$4.3 & 65.9$\pm$6.3 & 20.6$\pm$3.7 & 63.1$\pm$3.8 \\
GIN              & 89.4$\pm$5.6 & 82.7$\pm$1.7 & 64.6$\pm$7.0 & 75.3$\pm$2.9 & 44.5$\pm$4.1 & 72.8$\pm$3.6 \\
GraphSAGE        & 83.6$\pm$9.6 & 76.0$\pm$1.8 & 68.8$\pm$4.5 & 72.9$\pm$2.0 & 46.1$\pm$5.4 & 74.3$\pm$3.8 \\
KerGNN           & 81.5$\pm$0.4 & 82.8$\pm$1.8 & 70.5$\pm$0.8 & 77.6$\pm$3.7 & 59.5$\pm$4.5 & 75.8$\pm$3.5 \\
KAGNN            & 85.1$\pm$5.6 & 79.2$\pm$1.8 & 69.9$\pm$3.8 & 71.2$\pm$3.6 & 59.1$\pm$4.5 & 75.7$\pm$3.5 \\
\midrule
\textbf{NASK} & \textbf{95.9$\pm$3.6} & \textbf{88.0$\pm$4.5} & \textbf{76.7$\pm$2.6} & \textbf{82.1$\pm$0.6} & \textbf{67.7$\pm$6.5} & \textbf{82.6$\pm$1.3} \\
\bottomrule
\end{tabular}}
\end{table}

\section{Experiments}
In this section, we empirically evaluate NASK for attributed graph classification. Our experiments aim to answer:
\textbf{Q1.} Does NASK consistently achieve high accuracy across categorical, numerical, and heterogeneous attributed graphs?  
\textbf{Q2.} Can NASK maintain strong performance on large-scale real-world attributed graphs?  
\textbf{Q3.} How do the three key components of NASK contribute to its performance? 
\textbf{Q4.} \hh{What are the optimal WL iteration depth and star size for effectively capturing more structural features in NASK , and do these improvements come without incurring significant computational overhead compared to existing baselines?}

\subsection{Datasets and Experimental Setup} \label{Experiments}
We evaluate NASK on eleven benchmark datasets covering categorical (MUTAG~\cite{debnath1991structure}, NCI1~\cite{kriege2016valid}, PTC\_MR~\cite{toivonen2003statistical}, D\&D ~\cite{dobson2003distinguishing}, PROTEINS~\cite{borgwardt2005protein}), numerical (SYNTHETIC~\cite{kazius2005derivation}, SYNTHIE~\cite{morris2016faster}), and heterogeneous attributes (ENZYMES~\cite{togninalli2019wasserstein}, PROTEINS\_full~\cite{borgwardt2005protein}  BZR~\cite{sutherland2003spline}, COX2~\cite{sutherland2003spline}), as well as four large-scale real-world graphs benchmarks (Pubmed ~\cite{sen2008collective}, Cora~\cite{mccallum2000automating}, Citeseer~\cite{pottenger2012social}, Pokec~\cite{lim2021large}). 
Baselines include nine attributed graph kernels and seven widely used Graph Neural Networks (GNNs). Further details are in \textbf{Appendix} D.1 and D.2. All baselines and NASK are evaluated using 10-fold cross-validations repeated ten times on the same hardware. The best results in the tables are highlighted in \textbf{bold}, and accuracy is reported as the mean~$\pm$~standard deviation (\%). Details of the experimental setup are provided in \textbf{Appendix} D.3.

\begin{table*}[!t]
\centering
\begin{minipage}{0.48\textwidth}
\centering
\captionof{table}{\footnotesize Classification accuracy on real-world attributed graphs benchmarks.}
\label{tab:realworld-acc}
\scriptsize 
\setlength{\tabcolsep}{5pt}
{\renewcommand{\arraystretch}{0.9}
\begin{tabular}{lcccc}
\toprule
\textbf{Method} & Pubmed & Cora & Citeseer &Pokec  \\
\midrule
GCN           &84.00$\pm$3.54  & 85.10$\pm$2.57 & 77.30$\pm$1.30 & 75.45$\pm$0.10\\
GAT           & 84.95$\pm$1.77 & 87.70$\pm$0.30 & 76.20$\pm$0.80 &71.77$\pm$6.10  \\
KernelGCN        & 87.84$\pm$0.12 & 88.40$\pm$0.24 & 80.28$\pm$0.03 & 78.94$\pm$0.10 \\
AKGNN  &80.17$\pm$0.78  &83.95$\pm$0.26 &73.55$\pm$0.21   &78.83$\pm$0.10\\
KAGNN         & 71.74$\pm$0.29 & 81.72$\pm$0.79 & 68.16$\pm$0.61  &81.07$\pm$0.10 \\
\midrule
\textbf{NASK} & \textbf{89.53 $\pm$0.01} & \textbf{89.24$\pm$0.31} & \textbf{80.78$\pm$0.28}  &\textbf{83.05 $\pm$0.10}\\
\bottomrule
\end{tabular}}
\end{minipage}
\hspace{8pt}%
\begin{minipage}{0.48\textwidth}
\centering
\captionof{table}{\footnotesize Ablation study: classification accuracy of NASK and its variants.}
\label{tab: ablation}
\scriptsize
\setlength{\tabcolsep}{5pt}
{\renewcommand{\arraystretch}{0.9}
\begin{tabular}{lcccc}
\toprule
\textbf{Variant} & MUTAG & PTC\_MR & PROTEINS\_full & BZR \\
\midrule
\textbf{NASK}          & \textbf{95.9$\pm$3.6} & \textbf{76.7$\pm$2.7} & \textbf{81.1$\pm$0.9} & \textbf{88.8$\pm$1.8} \\
\midrule
w/ random walk  & 85.7$\pm$2.7 & 69.1 $\pm$1.5 & 73.5$\pm$1.5 & 82.7$\pm$0.2 \\
\midrule
w/ One-Hot  & 89.5$\pm$4.3 & 71.0$\pm$0.4 &76.4$\pm$2.2  & 85.8$\pm$2.0 \\
w/ Euclidean & 89.9$\pm$4.3 & 72.1$\pm$1.5 & 78.9$\pm$1.8 & 85.3$\pm$1.3 \\
\midrule

w/o WL Iteration & 89.2$\pm$0.3 & 71.9$\pm$1.4 & 77.6$\pm$1.4 & 84.2$\pm$0.4 \\

\bottomrule
\end{tabular}}
\end{minipage}
\end{table*}

\subsection{Evaluation on Diverse Attributed Datasets}
We evaluate NASK on eleven benchmark datasets with categorical, numerical, and heterogeneous attributes to answer \textbf{Q1}. As shown in Table~\ref{tab: accuracy2} and Table~\ref{tab: accuracy}, \textit{NASK consistently outperforms state-of-the-art graph kernels and GNNs across diverse attributed datasets.} \hh{Among these, on the PTC\_MR dataset, NASK outperforms the best existing method by 6.8\% in accuracy. On other numerical and heterogeneous attributed datasets (e.g., PROTEINS\_full), our model achieves a 2.4\% performance improvement.} These results highlight NASK’s effectiveness and generalization ability across heterogeneous attributed graphs, \hh{with the our proposed attribute similarity $\mathsf{P}$ enabling stable modeling of diverse attributes.}

\subsection{Evaluation on Large-Scale Real-World Datasets}
To answer \textbf{Q2}, we evaluate NASK on four real-world graphs with diverse attributes, including large-scale (Pubmed), citation-based (Cora, Citeseer), and social network (Pokec). As shown in Table~\ref{tab:realworld-acc}, \textit{NASK consistently achieves the best performance on all benchmarks.} Overall, NASK achieves up to 2.0\% higher accuracy, compared to the best-performing baselines, consistently outperforming them on four large-scale real-world benchmarks. These results demonstrate the robustness and effectiveness of NASK in capturing both structural and attribute information.

\subsection{Ablation Study}\label{sec: E3}
\hh{To address \textbf{Q3}, we conducted ablation experiments on NASK’s components, focusing on star subgraphs, attribute similarity $\mathsf{P}$, and WL iteration. As shown in Table~\ref{tab: ablation}, replacing the star subgraph with a random walk leads to performance drops of 6.1\% to 10.2\% across different datasets. This highlights the importance of fixed-center representations (i.e., star subgraph) for capturing richer semantic relationships. On datasets with numerical and heterogeneous attributes (e.g., PROTEINS\_full), NASK outperforms one-hot encoding by 3.7\% and Euclidean distance by 2.2\%. The attribute similarity $\mathsf{P}$ enables better modeling of both numerical and categorical features, unlike one-hot encoding and Euclidean distance, which fail to capture category relationships and struggle with mixed data types. Moreover, removing the WL iteration causes a 3.5\% to 6.7\% performance drop, confirming that relying on one-hop neighborhoods limits the model’s ability to capture complex structures. This highlights the need for multi-hop refinement to enhance expressiveness and ensure accurate predictions.}

\begin{figure}[t]
    \centering
    \subfloat[Accuracy with NASK of different depths of WL iteration \label{fig:wl_acc}]{
        \includegraphics[width=0.45\textwidth]{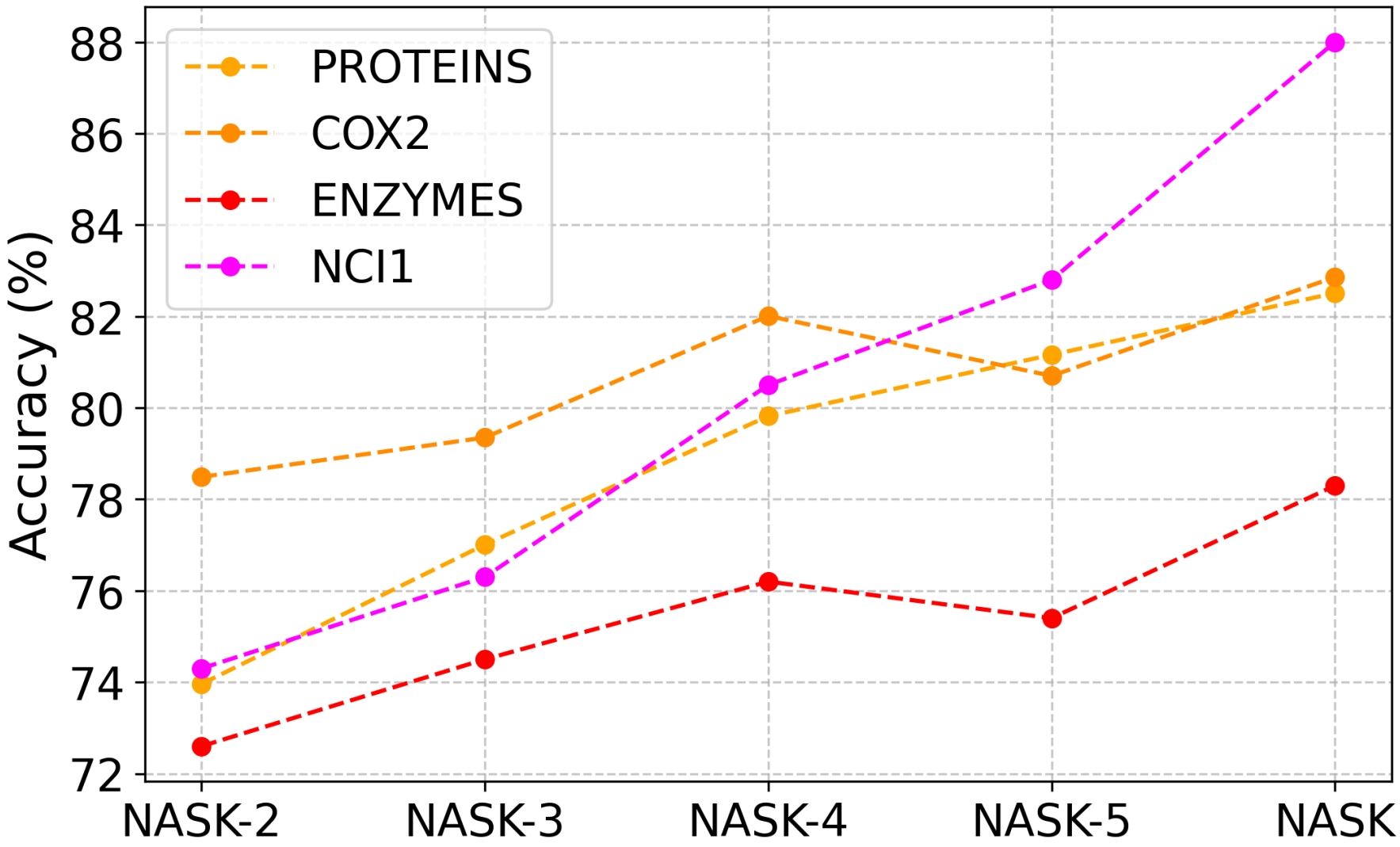}
    }
    \hfill
    \subfloat[Runtime with NASK of different depths of WL iteration \label{fig:wl_time}]{
        \includegraphics[width=0.45\textwidth]{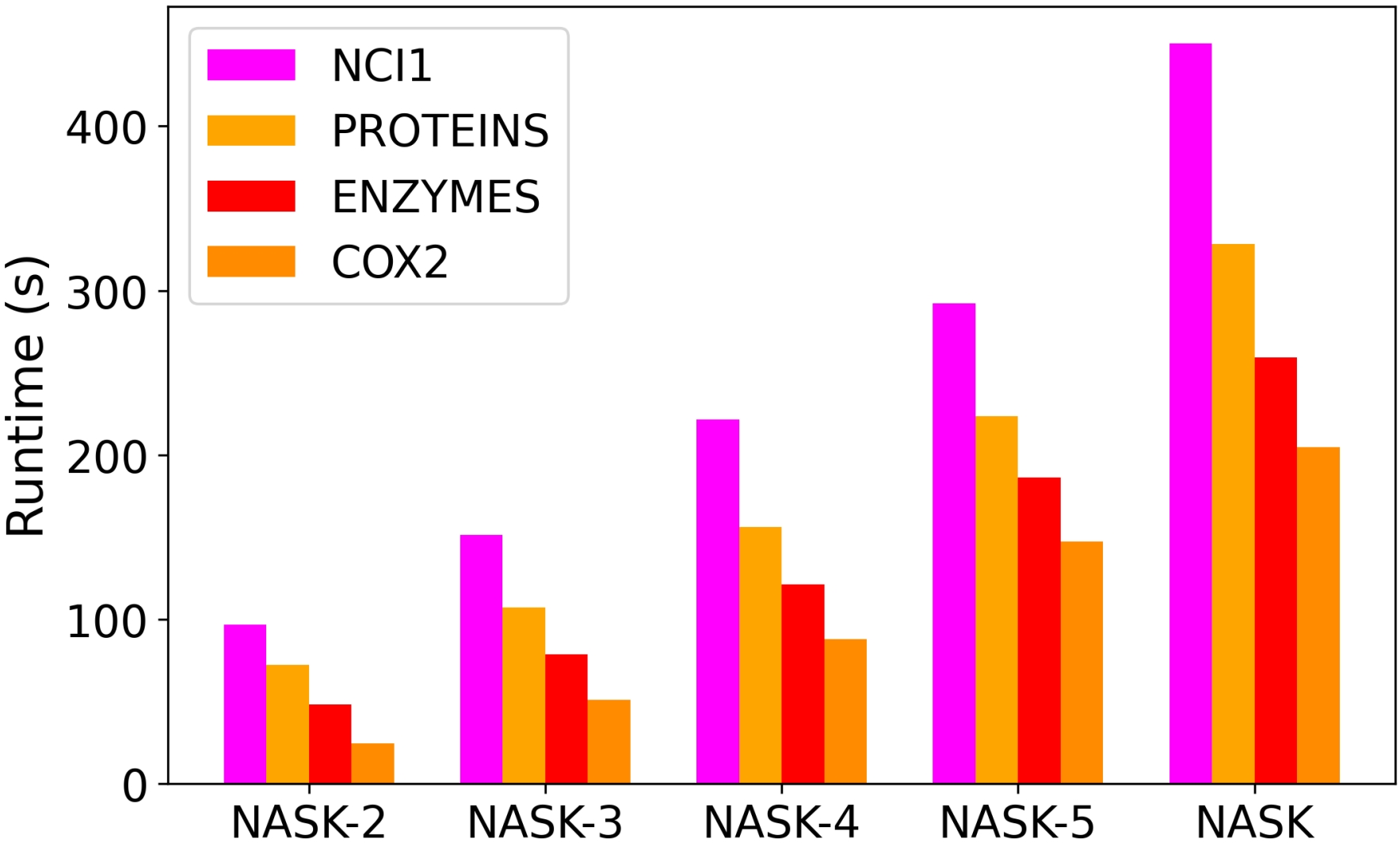}
    }
\caption{Impact of depths of WL iteration on classification accuracy and runtime of NASK across four datasets.}
    \label{fig:wl_multi}
\end{figure}

\subsection{Sensitivity to Depths of WL Iteration}
\hh{In this experiment, we evaluate NASK with varying depths of WL iteration across four datasets to answer \textbf{Q4}. In the Figure~\ref{fig:wl_multi}, NASK-$h$ denotes the variant with depth limited to $h$, while NASK without a subscript refers to the default configuration. Figure~\ref{fig:wl_acc} shows that increasing depth generally improves accuracy, with NASK-5 achieving 88.0\% on NCI1 and 82.5\% on PROTEINS. However, the improvement diminishes or even decreases after a certain depth, particularly on ENZYMES and COX2. Figure~\ref{fig:wl_time} shows that runtime increases with deeper WL iterations, especially from NASK-4 to NASK-5. On NCI1, runtime grows by 28.7\%, and on PROTEINS, by 41.8\%. The runtime is measured as the total time to compute the Gram matrix, including pairwise similarity calculation between graphs. Despite this, NASK-4 offers a good balance between accuracy and efficiency, making it ideal for computationally constrained environments. Furthermore, we conducted runtime comparisons with baselines to demonstrate that NASK achieves competitive efficiency, as detailed in Appendix E.2.}

To better understand NASK’s classification behavior, we perform a qualitative case study on the NCI1 dataset. Figure~\ref{fig:star_subgraph} visualizes $k$-hop star subgraphs ($k=1$ to $4$) centered on the highest-degree node in two correctly classified NCI1 graphs. At $k=1$, the star subgraph captures the immediate chemical environment, sufficient for simple functional groups. As $k$ increases, NASK captures larger substructures and relational dependencies, validating the effectiveness of our star subgraph extraction design. 

\begin{figure}[t]
    \centering
    \includegraphics[width=0.9\textwidth]{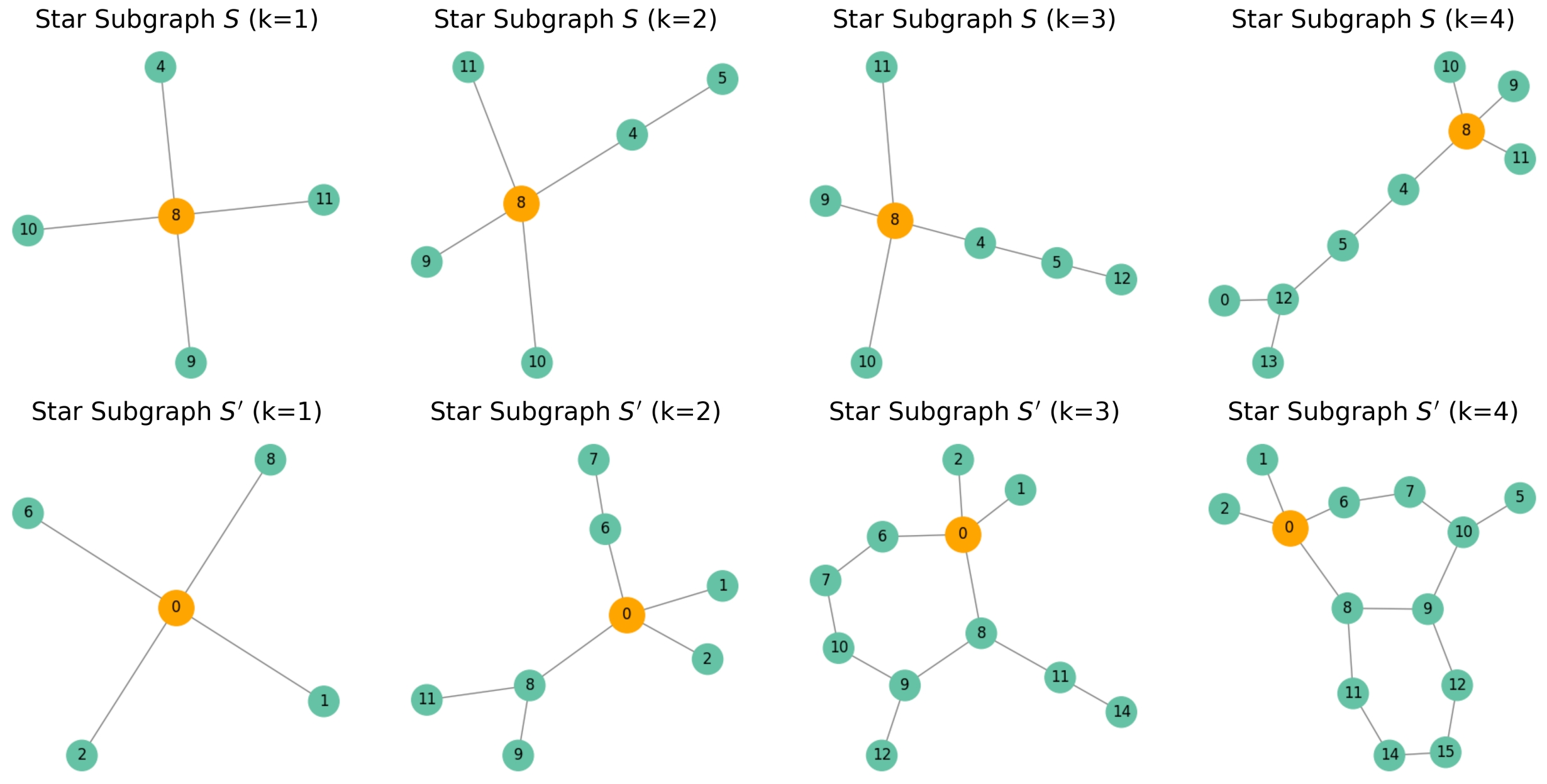}
    \caption{\footnotesize 
    Visualization of multi-hop star subgraphs centered on key nodes.}
    \label{fig:star_subgraph}
\end{figure}

\section{Conclusion}
We propose the \hh{Neighborhood-Aware Star Kernel (NASK)}, \hh{a novel graph kernel that jointly models heterogeneous attributes and neighborhood information}. By leveraging exponential transformation of the Gower similarity coefficient and WL-enhanced star subgraphs, NASK effectively captures rich semantic representations while guaranteeing positive definiteness. We theoretically establish its validity within kernel learning frameworks and empirically demonstrate its robustness through extensive experiments. NASK consistently outperforms strong baselines across diverse attributed graph benchmarks, demonstrating superior accuracy, scalability, and broad applicability on real-world graph classification tasks.

\newpage
\appendix
\section{Detailed Proofs}
\subsection{Detailed Proof of Lemma~\ref{lemma1}} \label{proof of f_d}
\begin{definition}[Conditionally Negative Definite Function]
A symmetric function $f: \mathcal{X} \times \mathcal{X} \rightarrow \mathbb{R}$ is said to be \emph{conditionally negative definite (CND)} if for any finite set $\{x_1, x_2, \dots, x_n\} \subset \mathcal{X}$ and any real coefficients $\alpha_1, \alpha_2, \dots, \alpha_n \in \mathbb{R}$ satisfying $\sum_{i=1}^n \alpha_i = 0$, the following inequality holds:
\begin{align*}\label{cnd}
    \sum_{i=1}^n \sum_{j=1}^n \alpha_i \alpha_j f(x_i, x_j) \leq 0.
\end{align*}
\end{definition}
\begin{proof}
We already have that $f_d$ is a symmetric function. To prove that $f_d$ is conditionally negative definite (CND), we need to show that for any finite set $\{x_1, \ldots, x_n\} \subset \mathcal{X}_d$ and real coefficients $\alpha_1, \ldots, \alpha_n \in \mathbb{R}$ satisfying $\sum_{i=1}^n \alpha_i = 0$, the following inequality holds:
\begin{align*}
\sum_{i=1}^n \sum_{j=1}^n \alpha_i \alpha_j f_d(x_i, x_j). \end{align*}

We prove the inequality as follows:
\begin{align*}
    \sum_{i=1}^n \sum_{j=1}^n \alpha_i \alpha_j f_d(x_i, x_j)
    &= \sum_{i=1}^n \sum_{j=1}^n \alpha_i \alpha_j \left(1 - s_d(x_i, x_j)\right) \\
    &= \sum_{i=1}^n \sum_{j=1}^n \alpha_i \alpha_j - \sum_{i=1}^n \sum_{j=1}^n \alpha_i \alpha_j s_d(x_i, x_j) \\
    &= 0 - \sum_{i=1}^n \sum_{j=1}^n \alpha_i \alpha_j s_d(x_i, x_j),
\end{align*}
where the first term equals to 0 due to the constraint $\sum_{i=1}^n \alpha_i = 0$.

Since $s_d(x_i, x_i) \ge 0$ for all $i = j$, and $s_d(x_i, x_j) = 0$ for $i \ne j$, the double sum reduces to:
\begin{align*}
\sum_{i=1}^n \sum_{j=1}^n \alpha_i \alpha_j s_d(x_i, x_j)
= \sum_{i=1}^n \alpha_i^2 s_d(x_i, x_i) \ge 0.
\end{align*}
Therefore, we conclude that:
\begin{align*}
\sum_{i=1}^n \sum_{j=1}^n \alpha_i \alpha_j f_d(x_i, x_j) \le 0.
\end{align*}

\end{proof}

\subsection{Detailed Proof of Lemma~\ref{lemma2}}\label{proof of $s'_d$}
\begin{proof}
We have established in Lemma~\ref{lemma1} that $f_d(x_d, x_d') = 1 - s_d(x_d, x_d')$ is CND. Then by the classical result from Berg et al.~\cite{berg1984harmonic}, the exponential transformation given by the following equation defines a positive definite kernel for all $\gamma > 0$.

\begin{align*}
   k(x_d, x_d') = \exp(-\gamma f_d(x_d, x_d')) = \exp\left( -\gamma (1 - s_d(x_d, x_d')) \right).
\end{align*}

Hence, $s_d'(x_d, x_d')$ is a valid positive definite kernel. 
\end{proof}

\subsection{Detailed Proof of Theorem~\ref{theorem 1}}\label{proof of P}
\begin{proof}
$\mathsf{P}$ is defined as:
\begin{align*}
    \mathsf{P}(v, v') = \frac{1}{|\mathrm{A}\cup \mathrm{A}'|} \sum_{d=1}^{|\mathrm{A}\cup \mathrm{A}'|} s_d'(\lambda(\mathrm{A}, v)_d, \lambda'(\mathrm{A}', v')_d).
\end{align*}
We have that each $s_d'$ is a positive definite kernel on the domain of the $d$-th attribute (Lemma~\ref{lemma2}). By the closure properties of positive definite kernels:

\begin{itemize}
    \item A finite non-negative linear combination of positive definite kernels is also positive definite.
    \item The sum $\sum_{d=1}^D s_d'(\cdot, \cdot)$ is thus positive definite.
    \item Scaling by a positive constant $1/D$ preserves positive definiteness.
\end{itemize}

Therefore, the function is a positive definite kernel on the attribute space of nodes $v$ and $v'$, since it is a normalized sum of positive definite kernels.

\end{proof}

\subsection{Detailed Proof of Theorem~\ref{theorem 2}}\label{proof of $k_s$}
\begin{proof}
The kernel $k_\mathrm{s}$ compares two star subgraphs by summing the pairwise similarities between their decomposed elements (vertices and edges). Specifically:
\begin{equation}
    k_\mathrm{s}(\mathrm{S}, \mathrm{S'}) = \sum_{n \in R^{-1}(\mathrm{S})} \sum_{n' \in R^{-1}(\mathrm{S}')}  \mathsf{P}(\mathrm{C},\mathrm{C'})\mathsf{P}(n, n').
\end{equation}
We already have that the function $\mathsf{P}(n, n')$ is positive definite in Theorem~\ref{theorem 1}, as it is computed via $s_d'$, which is positive definite given by Lemma~\ref{lemma2}.

\hh{A product of two positive definite kernels (i.e., $\mathsf{P}(\mathrm{C},\mathrm{C'})\mathsf{P}(n, n')$) remains positive definite, and a finite sum of positive definite kernels preserves this property. }

Therefore, $k_\mathrm{s}(\cdot, \cdot)$ is a valid positive definite kernel. \hh{This ensures that $k_\mathrm{s}$ inherits the positive definiteness required for valid kernel-based learning methods.}
\end{proof}

\subsection{Detailed Proof of Theorem~\ref{Theorem2}}\label{proof of $K_s$}
\begin{proof}
$K_{\mathrm{S}}$ is defined as:

\begin{align*}
K_\mathrm{S}(\mathrm{G}, \mathrm{G}') = \sum_{\mathrm{S} \in \hat{\mathcal{S}}(\mathrm{G})} \sum_{\mathrm{S}' \in \hat{\mathcal{S}}(\mathrm{G}')} k_\mathrm{s}(\mathrm{S}, \mathrm{S}'),
\end{align*}

where $k_\mathrm{s}(\mathrm{S}, \mathrm{S}')$ is the local star kernel proven to be positive definite in Theorem~\ref{theorem 2}.

Each $k_\mathrm{s}(\mathrm{S}, \mathrm{S}')$ is positive definite over the space of star subgraphs. The kernel $K_\mathrm{S}$ aggregates these values over all pairs of stars across the two graphs. A finite sum of positive definite kernels is itself positive definite, due to the closure properties of kernel functions.

Thus, $K_\mathrm{S}(\cdot, \cdot)$ is a valid positive definite kernel on the space of attributed graphs.
\end{proof}

\subsection{Weisfeiler–Lehman (WL) Algorithm}\label{WL Algorithm}
WL algorithm~\cite{Weisfeiler1968} is a color refinement algorithm that iteratively updates the labels of nodes to capture structural information. Initially, all nodes $v \in \mathrm{G}$ are assigned the same color: $\mathrm{WL}^0(\mathrm{G}, v) = 1, \quad \forall v \in \mathrm{G}.\mathrm{V}.$ At each iteration $h > 0$, the label of each node $v$ is updated based on its current label and the multiset of its neighbors labels: $
\mathrm{WL}^h(\mathrm{G}, v) = \operatorname{HASH}\left(\mathrm{WL}^{h-1}(\mathrm{G}, v), 
\operatorname{Multiset}\left\{\mathrm{WL}^{h-1}(\mathrm{G}, u) \mid u \in \mathcal{N}(v)\right\}\right),
$
where $\operatorname{HASH}(\cdot)$ is an injective hashing function. This process continues until the colors stabilize or a maximum iteration depth $H$ is reached. Two graphs $\mathrm{G}$ and $\mathrm{G}'$ are considered non-isomorphic if the multisets of node labels differ at any iteration.

\subsection{Detailed Proof of Theorem~\ref{theorem 4}}\label{proof of K_SWL}
\begin{proof}
At iteration $H$, the kernel $K_{\mathrm{NAS}}^{(H)}$ is defined as:

\begin{align*}
K_{\mathrm{NAS}}^{(H)}(\mathrm{G}, \mathrm{G}') = \sum_{h = 1}^{H}\sum_{\mathrm{S} \in \hat{\mathcal{S}}^{(h)}(\mathrm{G})} \sum_{\mathrm{S}' \in \hat{\mathcal{S}}^{(h)}(\mathrm{G}')} k_\mathrm{s}(\mathrm{S}, \mathrm{S}'),
\end{align*}

where $\hat{\mathcal{S}}^{(h)}(\mathrm{G})$ denotes the set of $h$-hop star subgraphs in $\mathrm{G}$, and $k_\mathrm{s}(s, s')$ is the local star kernel established in Theorem~\ref{theorem 2} to be positive definite.

Each $h$-hop star subgraph $\mathrm{S}^{(h)}$ is defined based on a fixed-radius neighborhood and preserves the underlying graph structure and attributes. The aggregation of $k_\mathrm{s}$ over all pairs of such subgraphs between two graphs defines a valid R-convolution kernel~\cite{haussler1999convolution}.

Since:
\begin{itemize}
  \item $k_\mathrm{s}$ is positive definite.
  \item The set of $h$-hop star subgraphs is finite for finite graphs.
  \item The triple sum is a finite sum of PD kernels.
\end{itemize}

It follows that $K_{\mathrm{NAS}}^{(H)}$ is a valid positive definite kernel due to the closure of positive definite kernels under summation.

\end{proof}

\subsection{Complexity Analysis}\label{Complexity Analysis}
We analyze the computational complexity of the proposed NASK kernel $K_{\mathrm{NAS}}^{(H)}$ in terms of graph size and WL iteration. Let $n$ and $m$ denote the number of nodes and edges in a graph, respectively, and let $d$ be the dimensionality of node or edge attributes. For a single graph $\mathrm{G}$, the number of extracted star subgraphs is $\mathcal{O}(n)$, and after $H$ iterations of WL expansion, the number of expanded star subgraphs remains $\mathcal{O}(n)$ assuming bounded neighborhood size. Each kernel computation $k_\mathrm{s}(\mathrm{S}, \mathrm{S}')$ between a pair of star subgraphs requires comparing all attribute pairs between two \hh{star graphs , yielding a cost of $\mathcal{O}((n+m)^2d)$ per pair}.

Given two graphs $\mathrm{G}$ and $\mathrm{G}'$, the NASK kernel sums over all pairs of expanded star subgraphs across $H$ iterations:
\begin{align*}
K_{\mathrm{NAS}}^{(H)}(\mathrm{G}, \mathrm{G}') = \sum_{h=1}^{H} \sum_{\mathrm{S} \in \hat{\mathcal{S}}^{(h)}(\mathrm{G})} \sum_{\mathrm{S}' \in \hat{\mathcal{S}}^{(h)}(\mathrm{G}')} k_\mathrm{s}(\mathrm{S}, \mathrm{S}').
\end{align*}
The total computational complexity is therefore $\mathcal{O}(Hn^2(n+m)^2d)$ in the worst case, assuming both input graphs have $\mathcal{O}(n)$ nodes and bounded degree. In practice, the computation is often faster due to early pruning based on core similarity thresholds (see Eq.~\ref{k_s}), which skips many redundant or dissimilar subgraph comparisons.

\section{Extension to Classification Tasks}\label{Classification Tasks}

The proposed \hh{Neighborhood-Aware Star Kernel} (NASK) naturally supports downstream classification tasks through its induced Reproducing Kernel Hilbert Space (RKHS) structure.

Given an attributed graph \(\mathrm{G} \in \mathcal{X}\), the NASK defines a positive definite kernel \(K_{\mathrm{NAS}}: \mathcal{X} \times \mathcal{X} \rightarrow \mathbb{R}\), corresponding to an implicit feature map
\begin{equation}
\varphi: \mathrm{G} \mapsto \varphi(\mathrm{G}) \in \mathcal{H},
\end{equation}
where \(\mathcal{H}\) is the RKHS associated with \(K_{\mathrm{NAS}}\), satisfying the reproducing property
\begin{equation}
K_{\mathrm{NAS}}(\mathrm{G}, \mathrm{G}') = \langle \varphi(\mathrm{G}), \varphi(\mathrm{G}') \rangle_{\mathcal{H}}.
\end{equation}

Given a training set \(\{(\mathrm{G}_i, y_i)\}_{i=1}^N\), where \(y_i \in \{+1, -1\}\), the regularized empirical risk minimization problem is formulated as:
\begin{equation}
\min_{f \in \mathcal{H}} \frac{1}{N} \sum_{i=1}^N \ell(f(\mathrm{G}_i), y_i) + \lambda \|f\|_{\mathcal{H}}^2,
\end{equation}
where \(\ell(\cdot, \cdot)\) is a convex loss function (e.g., hinge loss for SVMs), and \(\lambda > 0\) is the regularization parameter.

By the Representer Theorem, the optimal solution \(f^*\) admits a finite kernel expansion:
\begin{equation}
f^*(\cdot) = \sum_{i=1}^N \alpha_i K_{\mathrm{NAS}}(\mathrm{G}_i, \cdot),
\end{equation}
where \(\{\alpha_i\}_{i=1}^N \subset \mathbb{R}\) are the learned coefficients.

Substituting into the primal objective, we obtain the dual optimization problem:
\begin{align}
\max_{\boldsymbol{\alpha} \in \mathbb{R}^N} \quad \sum_{i=1}^N \alpha_i - \frac{1}{2} \sum_{i,j=1}^N \alpha_i \alpha_j y_i y_j K_{\mathrm{NAS}}(\mathrm{G}_i, \mathrm{G}_j) \\ \notag
\quad \text{subject to} \quad 0 \leq \alpha_i \leq C,\quad \sum_{i=1}^N \alpha_i y_i = 0,
\end{align}
where \(C\) is the SVM margin parameter.

The resulting classifier takes the following form:
\begin{equation}
f(\mathrm{G}) = \mathrm{sign}\left( \sum_{i=1}^N \alpha_i y_i K_{\mathrm{NAS}}(\mathrm{G}_i, \mathrm{G}) + b \right),
\end{equation}
where \(b\) is the learned bias.

Thanks to the positive definiteness of \(K_{\mathrm{NAS}}\), the induced RKHS \(\mathcal{H}\) is well-defined, guaranteeing the theoretical soundness of the learning process. Consequently, standard kernel-based methods such as Support Vector Machines (SVMs) can be directly applied for attributed graph classification tasks.

\section{Pseudocode}\label{Pseudo}
\FloatBarrier 
The Pseudocode of $K_\mathrm{S}$ and $K_{\mathrm{NAS}}^{(H)}$ are listed in Algorithm 1 and Algorithm 2.

\begin{algorithm}[h]\label{algorithm 1}
\caption{Local Graph Kernel $K_\mathrm{S}(\mathrm{G}, \mathrm{G}')$}
\label{alg:star_kernel}
\KwIn{Two attributed graphs $\mathrm{G}, \mathrm{G}'$}
\KwOut{Local graph kernel $K_\mathrm{S}(\mathrm{G}, \mathrm{G}')$}

$\hat{\mathcal{S}}(\mathrm{G}) \leftarrow$ Extract all star subgraphs from $\mathrm{G}$ using Eq.~\ref{it0}\;
$\hat{\mathcal{S}}(\mathrm{G}') \leftarrow$ Extract all star subgraphs from $\mathrm{G}'$\;

$K_\mathrm{S} \leftarrow 0$\;

\ForEach{$\mathrm{S} \in \hat{\mathcal{S}}(\mathrm{G})$}{
    \ForEach{$\mathrm{S}' \in \hat{\mathcal{S}}(\mathrm{G}')$}{
        $K_\mathrm{S} \leftarrow K_\mathrm{S} + k_\mathrm{s}(\mathrm{S}, \mathrm{S}')$\;
    }
}
\Return{$K_\mathrm{S}$}
\end{algorithm}

\begin{algorithm}[h]\label{algorithm 2}
\caption{Neighborhood-Aware Star Kernel $K_{\mathrm{NAS}}^{(H)}(\mathrm{G}, \mathrm{G}')$}
\label{alg:swl_kernel}
\KwIn{Graphs $\mathrm{G}, \mathrm{G}'$; max iteration $H$}
\KwOut{NASK kernel $K_{\mathrm{NAS}}^{(H)}(\mathrm{G}, \mathrm{G}')$}

$\hat{\mathcal{S}}^{(1)}(\mathrm{G}) \leftarrow$ Extract star subgraphs from $\mathrm{G}$\;
$\hat{\mathcal{S}}^{(1)}(\mathrm{G}') \leftarrow$ Extract star subgraphs from $\mathrm{G}'$\;

\For{$h = 1$ \KwTo $H$}{
    $\hat{\mathcal{S}}^{(h)}(\mathrm{G}) \leftarrow \{ \mathcal{L}(\mathrm{S}) \mid \mathrm{S} \in \hat{\mathcal{S}}^{(h-1)}(\mathrm{G}) \}$\;
    $\hat{\mathcal{S}}^{(h)}(\mathrm{G}') \leftarrow \{ \mathcal{L}(\mathrm{S}') \mid \mathrm{S}' \in \hat{\mathcal{S}}^{(h-1)}(\mathrm{G}') \}$\;
}

$K_{\mathrm{NAS}} \leftarrow 0$\;

\For{$h = 1$ \KwTo $H$}{
    \ForEach{$\mathrm{S} \in \hat{\mathcal{S}}^{(h)}(\mathrm{G})$}{
        \ForEach{$\mathrm{S}' \in \hat{\mathcal{S}}^{(h)}(\mathrm{G}')$}{
            $K_{\mathrm{NAS}} \leftarrow K_{\mathrm{NAS}} + k_\mathrm{s}(\mathrm{S}, \mathrm{S}')$\;
        }
    }
}
\Return{$K_{\mathrm{NAS}}$}
\end{algorithm}

\section{Details of Experiments}\label{appendix: experimental}

\subsection{Dataset Details}\label{appendix: datasets}
We provide a detailed description of all benchmark datasets used in our experiments, categorized by attribute type and application domain. A summary of the dataset statistics, including the number of graphs, attribute types, edge information, and class distribution, is presented in Table~\ref{tab: dataset}.

\paragraph{Categorical Attributed Graphs}\textbf{MUTAG} is a dataset of 188 aromatic compounds where each graph represents a molecule. Nodes correspond to atoms with categorical labels (e.g., C, O, N), and each graph is labeled based on mutagenicity. \textbf{NCI1} consists of 4,110 chemical compound graphs from the National Cancer Institute's screening database. Each node represents an atom with discrete labels, and the task is to classify compounds as active or inactive. \textbf{PTC\_MR} is a dataset of chemical compounds with 344 graphs. Each graph represents a molecular structure and is labeled based on carcinogenicity for rodents. \textbf{D\&D} is a dataset of 1,178 protein structures, where nodes represent amino acids and edges denote spatial closeness. \textbf{PROTEINS} is a dataset in which nodes are secondary structure elements. It has 3 discrete labels, which represent spirals, slices, or turns.

\paragraph{Numerical Attributed Graphs}\textbf{SYNTHETIC} is composed of 4,337 molecular graphs constructed for binary classification. Each node is described by a high-dimensional numerical vector, testing the ability to model rich continuous features. \textbf{SYNTHIE} is a synthetic dataset of 400 graphs with 4 class labels, designed to test sensitivity to both structural and attribute-based patterns. Each node has a 15-dimensional numerical feature vector and no categorical label.

\paragraph{Heterogeneous Attributed Graphs}\textbf{ENZYMES} contains 600 protein tertiary structure graphs labeled by enzyme class (6-way classification). Each node has a categorical label for structural role and 18 real-valued attributes for physico-chemical properties. \textbf{PROTEIN\_full} consists of protein graphs where nodes represent amino acids and edges indicate chemical interactions. Attributes are mixed with both structural and chemical properties. \textbf{BZR} and \textbf{COX2} contain graphs representing drug molecules. Each node (atom) has a discrete element label and continuous features such as partial charges.

\paragraph{Large-scale Real-world Attributed Graphs}\textbf{Pubmed} is a citation network with nodes representing papers and edges denoting citation links. Node features are TF-IDF weighted word vectors from paper abstracts. \textbf{Cora} is a citation graph whose nodes represent papers and edges are citation links. The graph consists of 2,708 papers and 5,278 edges, with 1,433 attribute dimensions. \textbf{Citeseer} is another citation network with 3,327 papers and 4,732 edges, including 3,703 attribute dimensions. \textbf{Pokec} is a large-scale social network graph where nodes represent users and edges correspond to friendship relationships. The graph contains rich user attributes such as age, gender, and interests, making it ideal for social influence modeling and community detection tasks.

\begin{table}[ht]
\small
\caption{Dataset Statistics}
\label{tab: dataset}
\setlength{\tabcolsep}{4pt}
\renewcommand{\arraystretch}{1.2}
\centering
\begin{tabular}{l c c c c}
\toprule
\textbf{Dataset} & \textbf{Node Attr.} & \textbf{Edge Attr.} & \textbf{\# Classes} & \textbf{\# Graphs} \\ 
\midrule
MUTAG           & Categorical     & -   & 2   & 188   \\ 
NCI1            & Categorical     & -   & 2   & 4110  \\ 
PTC\_MR         & Categorical     & -   & 2   & 344   \\ 
D\&D            & Categorical     & -   & 2   & 1178  \\ 
PROTEINS        & Categorical     & \checkmark   & 2   & 1113  \\ 
SYNTHETIC       & Numerical       & -   & 2   & 4337  \\ 
SYNTHIE         & Numerical       & -   & 4   & 400   \\ 
ENZYMES         & Heterogeneous           & \checkmark   & 6   & 600   \\ 
PROTEINS\_full  & Heterogeneous           & \checkmark   & 2   & 1113  \\ 
BZR             & Heterogeneous           & \checkmark   & 2   & 405   \\ 
COX2            & Heterogeneous           & \checkmark   & 2   & 467   \\ 
\midrule
Pubmed          & Real-world      & \checkmark   & 3   & 1     \\ 
Cora            & Real-world      & \checkmark   & 7   & 1     \\ 
Citeseer        & Real-world      & \checkmark   & 6   & 1     \\ 
Pokec & Real-world & \checkmark & 2 & 1     \\ 
\bottomrule
\end{tabular}
\end{table}

\subsection{Baselines Details}\label{appendix: baselines}
\paragraph{Attributed Graph Kernels}We evaluate NASK against nine state-of-the-art attributed graph kernels, including WL-VH~\cite{shervashidze2011weisfeiler}, PK~\cite{neumann2012efficient}, HGK~\cite{morris2016faster}, WWL~\cite{togninalli2019wasserstein}, GH~\cite{feragen2013scalable}, ML~\cite{kondor2016multiscale}, RWK~\cite{wijesinghe2021regularized}, SWWL~\cite{perez2024gaussian} and RetGK~\cite{zhang2018retgk}. These kernels effectively capture both structural and attribute-based similarities through various strategies. For instance, WL-VH and PK rely on neighborhood aggregation and iterative relabeling to encode local structural patterns, while HGK introduces hashing mechanisms to accelerate graph comparisons. Advanced methods like WWL and RetGK utilize optimal transport and return probabilities, respectively, to capture fine-grained similarities across graph nodes and edges, enhancing the expressiveness of attributed graph representations.

\paragraph{Graph Neural Networks (GNNs)}The GNN baselines include GCN~\cite{DBLP:conf/iclr/KipfW17}, GIN~\cite{xu2018powerful}, GAT~\cite{velivckovic2018graph}, GraphSAGE~\cite{hamilton2017inductive},
AKGNN~\cite{ju2022adaptive},  KerGNN~\cite{feng2022kergnns}, and KAGNN~\cite{han2024kagnn}. These models leverage message passing and kernel-based similarity alignment to enhance representation learning. GCN aggregates information from neighbors using spectral convolutions, while GIN focuses on injective aggregation to improve expressiveness. GraphSAGE introduces inductive learning by sampling neighbors, facilitating scalability to larger graphs. GAT employs attention for adaptive neighbor aggregation, while AKGNN extends this with kernel-based attention to better capture attribute and structural information. Kernel-based methods like KerGNN and KAGNN further incorporate similarity metrics to better align node representations, effectively capturing both structural and semantic relationships.

These baselines collectively cover a wide range of structural and attribute modeling graphs, and serve as strong comparators for validating the effectiveness of NASK.

\subsection{Details of Experimental Setup }\label{appendix: hyper}

\paragraph{Configuration Details of Baselines and Our Method}
We evaluate NASK against a wide range of state-of-the-art baselines. For attributed graph kernels (e.g., WL, WWL, RetGK), we use the GraKeL library or official implementations, with hyperparameters set either by default or selected through small-scale validation. For classification, we adopt scikit-learn's C-SVM, tuning $C$ over ${10^{-3}, 10^{-2}, \dots, 10^{3}}$ via nested cross-validation. 

Graph neural network baselines (GCN, GraphSAGE, GIN, etc.) are implemented in PyTorch Geometric 2.0 with a uniform 2-layer architecture, each with 64 hidden units, ReLU activation, global sum pooling, and a softmax output layer. Models are trained for a maximum of 100 epochs using early stopping (patience = 10) and a validation split of 10\% from the training set. The learning rate is selected from $\{0.001, 0.005, 0.01\}$.

\paragraph{Evaluation Protocol}All models are evaluated using 10-fold cross-validations repeated ten times (30 randomized splits per dataset). We report the mean classification accuracy along with standard deviation. For multi-class datasets (e.g., ENZYMES, SYNTHIE), overall accuracy is used. Runtimes are reported in wall-clock time. All computations, including kernel construction and GNN training, are conducted on a machine equipped with an NVIDIA RTX 4090 GPU.

\section{Additional  Experiments}

\subsection{Comprehensive Ablation Analysis}
To evaluate the contribution of each component in our proposed NASK kernel, we conduct a systematic ablation study on several additional datasets. Figure~\ref{fig: ablation} visualizes the classification accuracy of NASK and its ablation variants across datasets with different attribute types. Specifically, datasets with simple categorical attributes are marked with solid symbols, while those with numerical and heterogeneous attributes use hollow markers. This visual distinction highlights how each variant performs under varying attribute complexities. We consider the following settings:

\textbf{w/ Random Walk}: replaces our star subgraph with a random walk-based structure;

\textbf{w/ One-Hot} and \textbf{w/ Euclidean}: simplify the attribute similarity function $\mathsf{P}$;

\textbf{w/o WL Iteration}: removes the iterative structural refinement via Weisfeiler-Lehman (WL) iterations.

\begin{figure}[htbp]
\centering
\includegraphics[width=\linewidth]{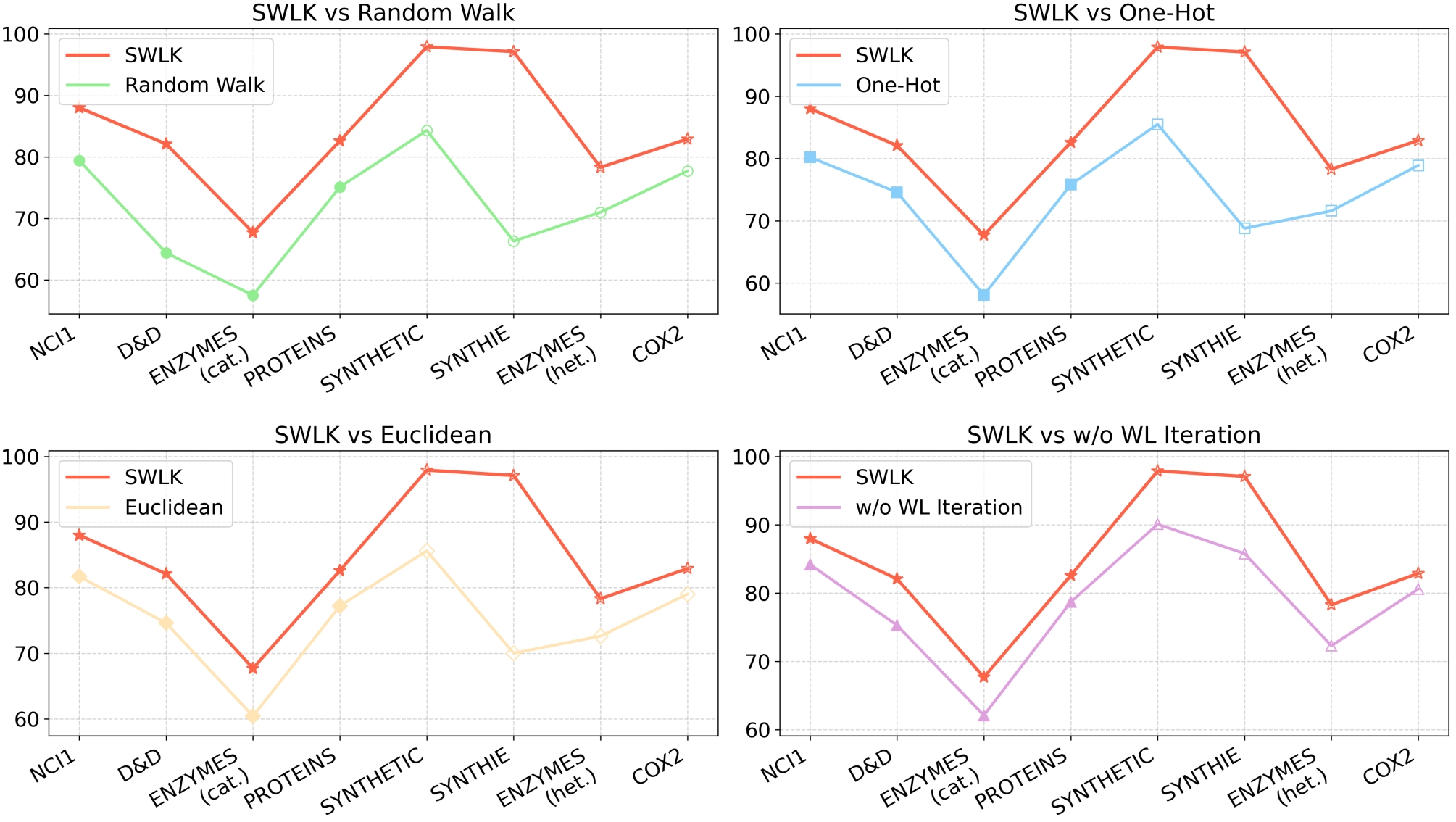}
\caption{\footnotesize Performance of NASK and its ablation variants on several additional datasets}
\label{fig: ablation}
\end{figure}

\paragraph{\textbf{Effect of Star Subgraphs}}
On structurally rich datasets such as D\&D and ENZYMES (cat.), replacing the star subgraph with a random walk results in performance drops of up to 21.5\%. This highlights the importance of using fixed-center representations that preserve local topology for capturing meaningful structural patterns.

\paragraph{\textbf{Effect of Attribute Similarity $\mathsf{P}$}}
On datasets with numerical and heterogeneous features (e.g., SYNTHIE, COX2), our adaptive attribute similarity significantly outperforms both one-hot and Euclidean variants. For example, on COX2, NASK improves classification accuracy by 4.0\% over one-hot and 3.9\% over Euclidean, demonstrating the expressiveness and adaptability of the similarity function $\mathsf{P}$.

\paragraph{\textbf{Effect of WL Iteration}}
Removing WL iteration consistently degrades performance across all datasets. The 4.7\% accuracy drop on PROTEINS demonstrates that relying solely on one-hop neighborhoods limits the model’s ability to encode complex structural information. This performance gap emphasizes the necessity of our multi-hop refinement design for enhancing expressiveness and ensuring accurate predictions.

Overall, NASK consistently achieves the highest accuracy across all datasets, highlighting the importance of its three core components: star subgraph extraction, attribute similarity computation, and WL iteration. These components are designed to work together to ensure strong performance, as removing any one of them leads to a significant drop in accuracy. The performance gaps between NASK and its variants are more pronounced on datasets with complex attributes, confirming the essential role of each component in enabling NASK to generalize across a wide range of graph types.

\subsection{Runtime Comparison with Baseline Methods}

Table~\ref{tab:accuracy2} presents the runtime comparison across six attributed graph benchmarks. The results reveal clear distinctions in computational efficiency between methods. Among all baselines, GH and ML exhibit the highest computational cost, often requiring multiple hours to process larger datasets such as NCI1 and PROTEINS\_full. For instance, ML takes over 7 hours on NCI1 and over 8 hours on MUTAG, indicating poor scalability when applied to even relatively small graphs with rich structure or attributes.

In contrast, NASK maintains a consistently low runtime across all datasets. Its runtime remains below 1 second on small datasets like MUTAG and PTC-MR, and scales moderately on larger datasets, requiring only 12 minutes on NCI1 and under 1 minute on PROTEINS\_full. This is in stark contrast to the steep growth seen in GH and ML. Compared to lightweight baselines such as PK, RetGK, and RWK, NASK remains competitive, sometimes slightly slower on the smallest datasets but generally faster on medium and large ones. Notably, on ENZYMES and BZR, NASK outperforms several baselines including GH, ML, and WWL, suggesting that it not only scales well but also adapts effectively to graphs with complex attributes or heterogeneity. Overall, the results demonstrate that NASK achieves improved performance without introducing significant runtime overhead, validating its efficiency and scalability in practical scenarios.

\begin{table*}[!t]
\setlength{\tabcolsep}{10pt}
\centering
\caption{\footnotesize Runtime on diverse attributed graphs benchmarks. All times are reported in hours, minutes and seconds.}
\label{tab:accuracy2}
\small
\begin{tabular}{l cccccc}
\toprule
\textbf{Method} & \textbf{MUTAG} & \textbf{NCI1} & \textbf{PTC-MR} & \textbf{PROTEINS\_full} & \textbf{ENZYMES} & \textbf{BZR} \\
\midrule
GH      & 24.79s      & 3h 43m 7.2s   & 1m 33.9s    & 3h 15m 5.36s   & 16m 36.48s   & 4m 11.74s  \\
PK      & 0.53s       & 10m 30.02s    & 1.79s       & 1m 10.25s      & 14.93s       & 7.24s      \\
ML      & 8h 16m 16.84s & 7h 18m 35.72s & 22m 9.56s   & 2h 48m 24.59s  & 1h 4m 52.54s & 47m 44.62s \\
WWL     & 3m 42s      & 16m 10s       & 1m 25s      & 13m 23.7s      & 3m 40.5s     & 1m 37.3s   \\
RetGK   & 0.9s        & 13m 3s        & 1.4s        & 2m 57s         & 37s          & 17s        \\
RWK     & 7s          & 12m 49s       & 3.6s        & 1m 49s         & 2m 19s       & 48s        \\
\midrule
\textbf{NASK} & 0.61s   & 12m 0.92s     & 0.98s       & 59.01s         & 13.89s       & 10.39s     \\
\bottomrule
\end{tabular}
\end{table*}

\subsection{Robustness to Node Attribute Perturbation}
We evaluate the robustness of NASK against noisy or missing node attributes on real-world attributed graphs. Although NASK is a deterministic kernel method and does not rely on gradient-based training like GNNs, its performance may still degrade when node attributes are corrupted. We conduct perturbation analysis on three molecular datasets with heterogeneous node attributes—BZR, COX2, and PROTEINS\_full. Two types of attribute noise are introduced:

\begin{itemize}
    \item \textbf{Gaussian noise addition}: Each node attribute \( x \) is perturbed as \( x' = x + \varepsilon \), where \( \varepsilon \sim \mathcal{N}(0, \sigma^2) \).
    \item \textbf{Feature masking}: A random subset of node attribute dimensions is set to zero for 10\%, 20\%, or 30\% of the nodes.
\end{itemize}

For each noise level \( \gamma \in \{0\%, 10\%, 20\%, 30\%\} \), we generate 10 perturbed datasets using different random seeds. For each setting, we compute the NASK kernel matrices and perform 10-fold cross-validations using a linear SVM. Classification accuracy is reported across multiple runs to evaluate performance under noise perturbations.

As illustrated in Figure~\ref{fig:robust-bzr-cox2}, NASK exhibits notable robustness under both types of perturbation. Under Gaussian noise (Figure~\ref{fig:robust-gaussian}), the model consistently maintains high accuracy across datasets: BZR shows minimal degradation, with median accuracy remaining above 86\%, and COX2 retains accuracy above 77\%. On PROTEINS\_full, the accuracy declines more gradually, likely due to attribute sparsity. In comparison, under feature masking (Figure~\ref{fig:robust-masking}), accuracy decreases somewhat more significantly, indicating that NASK tolerates continuous perturbations better than discrete information loss. Additionally, the narrow interquartile ranges across all settings underscore the stability and resilience of NASK to noisy or incomplete attribute information.

\begin{figure}[htbp]
\centering
\begin{subfigure}[t]{0.49\linewidth}
    \centering
    \includegraphics[width=\linewidth]{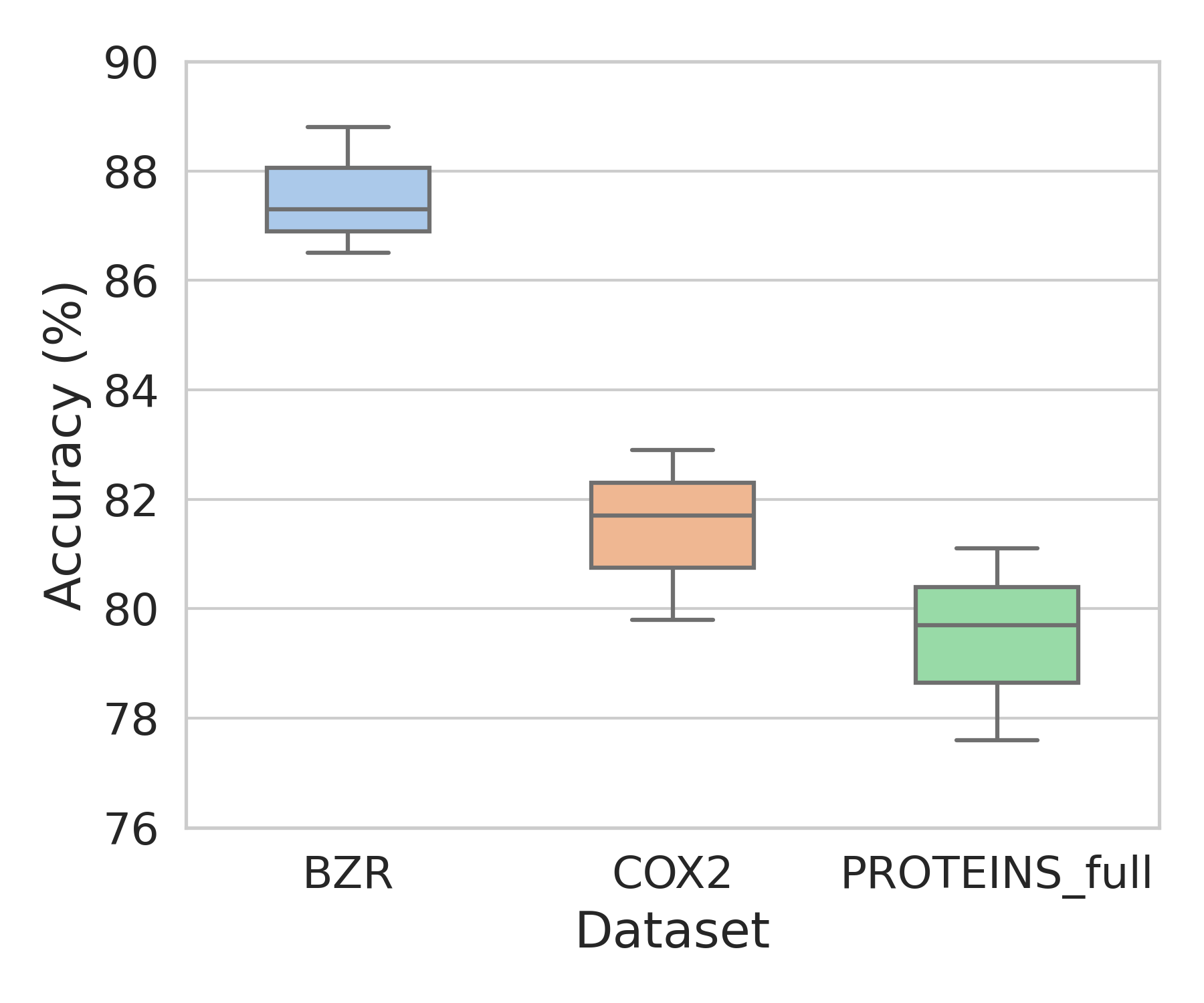}
    \caption{\footnotesize NASK under Gaussian Noise}
    \label{fig:robust-gaussian}
\end{subfigure}
\hfill
\begin{subfigure}[t]{0.49\linewidth}
    \centering
    \includegraphics[width=\linewidth]{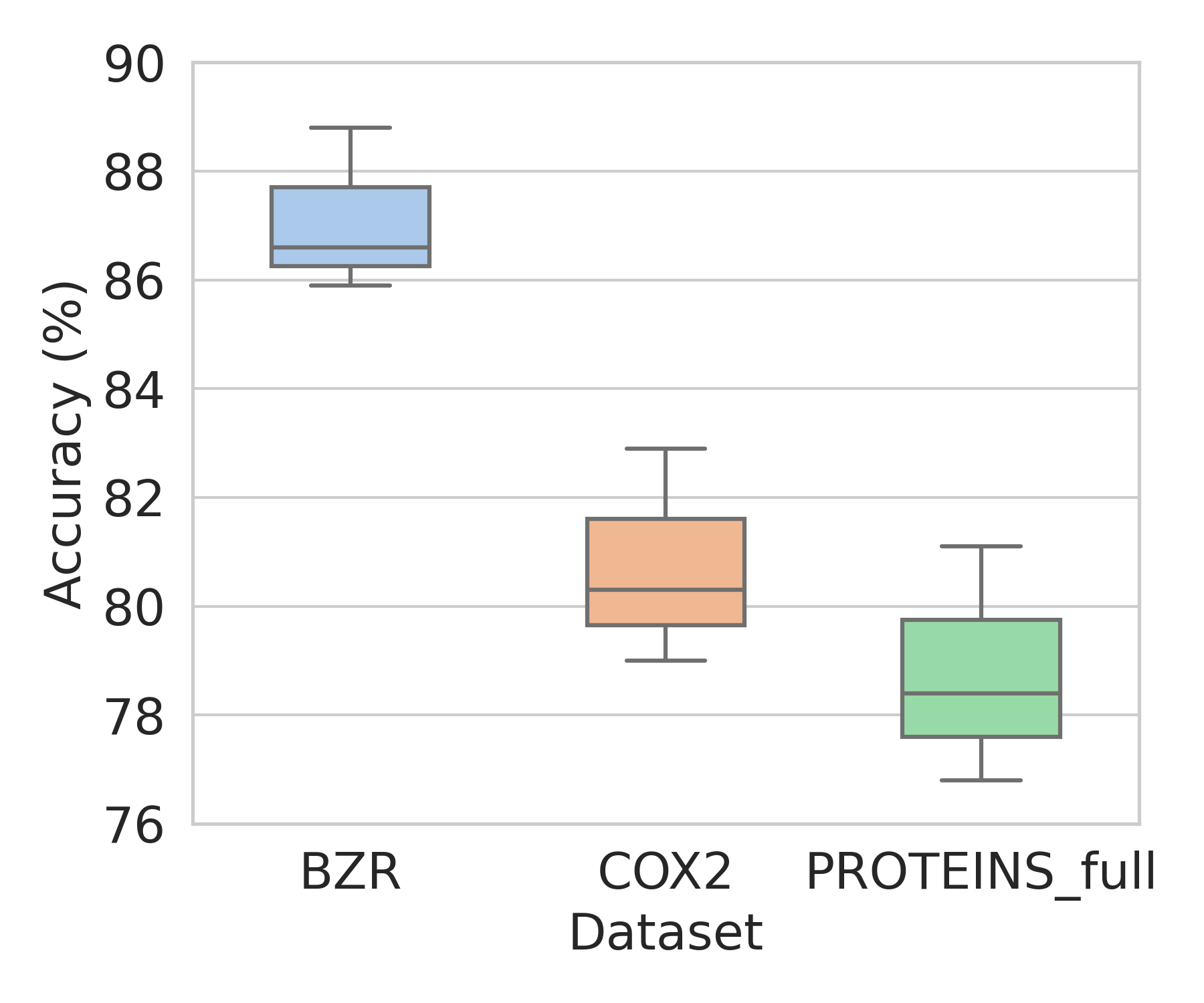}
    \caption{\footnotesize NASK under Feature Masking}
    \label{fig:robust-masking}
\end{subfigure}
\caption{\footnotesize Accuracy of NASK under attribute perturbations.}
\label{fig:robust-bzr-cox2}
\end{figure}

\subsection{Case Study on the NCI1 Dataset}
\begin{figure}[H]
    \centering
    \includegraphics[width=0.5\textwidth]{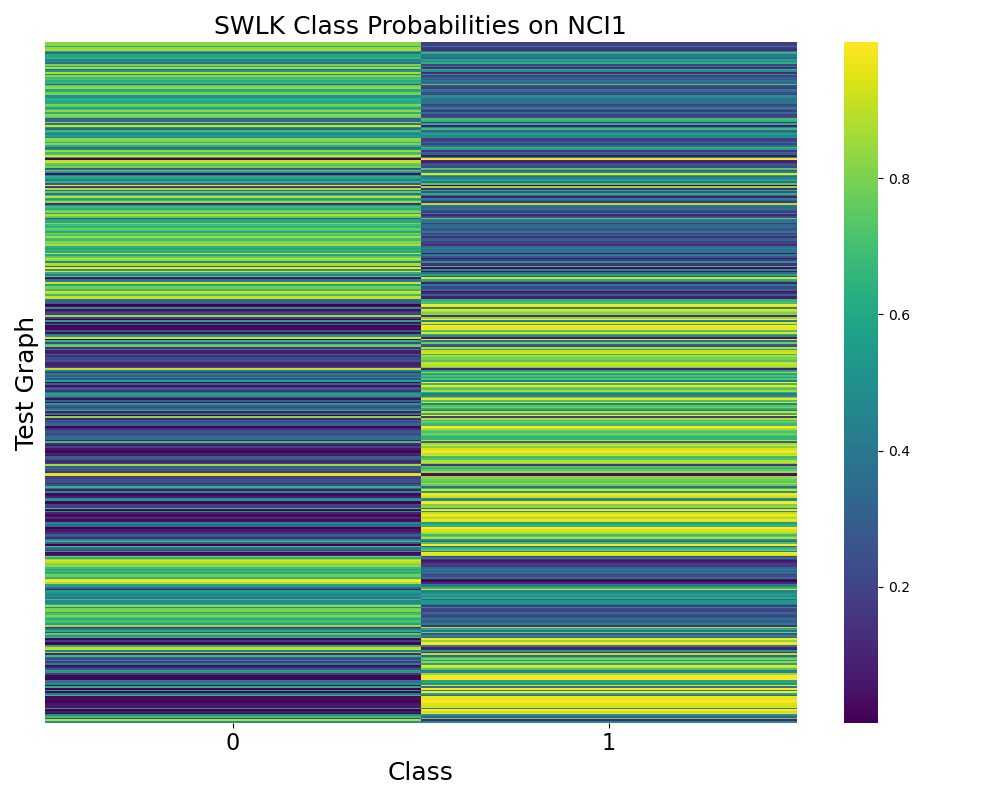}
    \caption{NASK class probability heatmap on the NCI1 dataset.}
    \label{fig:nci1_case}
\end{figure}
To better understand NASK’s classification behavior, we perform a qualitative case study on the NCI1 dataset. As shown in Figure~\ref{fig:nci1_case}, we analyze the predicted class probabilities. The visualization demonstrates how NASK integrates local structural patterns and attribute similarity to produce robust and interpretable predictions.

Figure~\ref{fig:nci1_case} presents a class probability heatmap, where each row corresponds to a test molecular graph and each column to one of the two predicted classes. The color intensity reflects the model’s confidence. Strong vertical bands indicate that NASK tends to assign high probabilities to one class while suppressing the other, demonstrating its discriminative capability. Misclassifications and low-confidence predictions appear as rows with ambiguous shading, yet remain relatively rare. This pattern suggests that NASK is not only accurate but also confident and consistent in its predictions across structurally diverse graphs.

\newpage
\bibliographystyle{unsrt}  
\bibliography{references}

\end{document}